%% file: ms.tex
\documentclass{article}

\usepackage[preprint]{nips_2018}

\usepackage[utf8]{inputenc} % allow utf-8 input
\usepackage[T1]{fontenc}    % use 8-bit T1 fonts
\usepackage{hyperref}       % hyperlinks
\usepackage{url}            % simple URL typesetting
\usepackage{booktabs}       % professional-quality tables
\usepackage{amsfonts}       % blackboard math symbols
\usepackage{nicefrac}       % compact symbols for 1/2, etc.
\usepackage{microtype}      % microtypography
\usepackage{algorithmic}
\usepackage{algorithm}
\usepackage{amssymb}
\usepackage{booktabs}
\usepackage{bm}
\usepackage{color}
\usepackage{amsmath}
\usepackage[capitalise]{cleveref}
\usepackage{multirow}
\usepackage{graphicx}
\usepackage{subfigure}

\newtheorem{definition}{\textbf{Definition}}
\newtheorem{lemma}{\textbf{Lemma}}
\newtheorem{theorem}{\textbf{Theorem}}

\newtheorem{assumption}{\textbf{Assumption}}

\newenvironment{proof}{\paragraph{Proof:}}{\hfill$\square$}

\newif\ifanon
\anonfalse

\ifanon
\author{
Anonymous Author(s)\\
Affiliation\\
Address\\
\texttt{email}
}
\else
\author{
  Jiaxiao Zheng, Gustavo de Veciana\\
  Department of Electrical and Computer Engineering\\
  The University of Texas at Austin\\
  Austin, TX 78712 \\
  \texttt{gustavo@ece.utexas.edu} \\
}
\fi

\begin{document}
% \nipsfinalcopy is no longer used
\title{Modeling and Optimization of Human-Machine Interaction Processes via the Maximum Entropy Principle}

\maketitle

\newif\ifappendix
\appendixtrue
%\appendixtrue

%\newif\ifextended
%\extendedfalse

%\newif\ifregularizedAREA
%\regularizedAREAfalse
%\regularizedAREAtrue

\begin{abstract}
We propose a data-driven framework to enable the modeling
and optimization of human-machine interaction processes,
e.g., systems aimed at assisting humans in decision-making
or learning, work-load allocation, and interactive advertising.
This is a challenging problem for several reasons.
First, humans' behavior is hard to model or
infer, as it may reflect biases, long term memory, and
sensitivity to sequencing, i.e., transience and exponential complexity
in the length of the interaction.
Second, due to the interactive nature of such processes,
the machine policy used to engage with human
may bias possible data-driven inferences.
Finally, in choosing machine policies that optimize interaction rewards,
one must, on the one hand, avoid being
overly sensitive to error/variability in the estimated
human model, and on the other, being overly
deterministic/predictable which may result in poor
human `engagement' in the interaction.
To meet these challenges, we propose a robust
approach, based on the maximum entropy principle, which
iteratively estimates human behavior and optimizes the machine
policy--Alternating Entropy-Reward Ascent (AREA) algorithm.
We characterize AREA, in terms of its space and time complexity and
convergence.  We also provide an initial validation based on synthetic
data generated by an established noisy nonlinear model for
human decision-making.
\end{abstract}

\input{introduction}
\input{formulation}
\ifappendix
\input{related-work-simplified}
\fi
\input{solution}
\input{complexity}
\input{AREA}

\input{decomposable-simplified}
\input{evaluation}

\bibliographystyle{unsrt}
\bibliography{./bib/entropy_MDP,./bib/gdv,./bib/interactive,./bib/optim,./bib/koyejo,./bib/art}

\ifappendix
\input{appendix}
\fi

\end{document}

%% file: introduction.tex
\section{Introduction}
%Background and motivations:
%\begin{enumerate}
%\item Importance and examples of human-machine interactions and information sequencing: education, advertising, robot, etc.
%\item Why maximal entropy principle, and moment matching techniques? What's the causality in the modeling?
%\item Why jointly estimating-optimizing?
%\item Brief description of the paper
%\end{enumerate}

Computing and information systems are
increasingly prevalent in our daily lives, which forms a variety of human-in-the-loop systems.
Many such systems are interactive in the sense that, humans and machines
take decisions/actions in response to each other, forming a sequence
driven by unknown dynamics associated with human behavior.
For instance, one can view web searches as an interactive process, where
humans' search history, attention, and eventual decisions reflect
an interaction with the machine's sequencing, placement and timing of advertisements.
The industry refers to such interactive processes as `convergence paths' and is increasingly
interested in optimizing their outcomes \cite{HCE16}.
%In an educational context, the sequencing of content, assignments, and their difficulties
%can be viewed as a machine's decisions, which may affect students' engagement and learning outcomes.
%Human decision-making, learning and cognition may depend critically on the way
%computing and information systems interact with humans.
Such problem involving interactions are usually studied under
the context of Markov decision processes (MDP) and 
its variants, see, e.g.,
\cite{Put94,Ber00,BaR11}. However, the actual problem associated with interactive
processes presents several challenges which remain unsolved, including the following.

%The aim of this paper
%is to explore how such systems might be improved if one wishes to emphasize
%the modeling and optimization of such interactive processes. This problem presents several
%challenges including the following.

\emph{Complexity of inferring interactive human behavior.}
In this paper we will focus on structured human-machine interactions where one has
modeled both human and machine behaviour/choices over time, and the setting
arises repeatedly either by the same person or by a large population.
%We recognize, however, that humans' behavior may change over time, e.g., as an individual
%or population become more adept with a given technology.
The outcomes of such interactions can depend subtly on the history thus one
can expect exponential complexity to be a challenge--unless
the underlying processes have a `nice' structure.
Such assumption is essential for widely studied problems including
MDP \cite{Put94,Ber00,BaR11}, reinforcement learning \cite{KLM96, SB98},
and multi-armed bandit problem \cite{BuC12}, where human decision-making processes
are assumed to be independent across time, or have one-step
Markov property. However, those assumptions are questionable according to
studies on human cognition, see, e.g., \cite{MuB62}.

One recent work considering long-term dependency is deep Q-learning \cite{VKD15},
where authors used a complex neural network to capture
the potential value of a state-action pair,
where the state may incorporate complex historical information.
However, because most interactive processes are transient, as both human and machine
accumulate a history of decisions over time,
one might expect the data requirements of carrying out deep Q-learning is quite high.

%{\color{red} Cognitive biases part deleted.}
%\emph{Cognitive biases.} Human decision-making is complex, and may depend on many different exogenous factors in
%addition to the task at hand. Of great importance, when humans are in-the-loop, is properly capturing the impact
%of typical cognitive and/or perceptual biases. Biases such as anchoring, confirmation, framing, base rate fallacy,
%primacy, and recency, among others, are well documented in the literature, see e.g., \cite{Kah11}, as is
%the notion of ``nudging" humans towards better decisions, see e.g., \cite{ThS09}. In our thesis, we consider such
%biases as \emph{features} of an interactive process that impacts human decision-making.

\emph{Biases in collecting data in interaction processes.}
Inferring a model for human behavior in the context of human-machine
interaction process is also challenging
because to collect data one must choose a machine policy to `interact'
with humans. This may in turn lead to `biased'
inferences of human behavior. In particular, a machine policy that focuses on
`rewarding' actions may preclude exploration/observation of other
interaction modalities.
Similar consideration was explored in partially observable MDP \cite{MeR05}
to improve the efficiency of the solution, and in multi-armed bandit problem
\cite{BuC12, AHK14} to achieve better exploration-exploitation tradeoff.
However, data-driven models and inferences should respect the
\emph{causal} nature of human-machine interactions, but how to
model/promote the randomness of a \emph{causal} model remains unknown.
%Furthermore, data-driven models and inferences should respect the
%\emph{causal} nature of human-machine interactions.

\emph{Robustness and exploration in optimizing machine interactions.}
A general data-driven framework for modeling and optimizing human-machine interaction
processes might be viewed as involving two concerns. On one hand, engaging humans in interaction to
collect data to infer models of human behavior, and on the other,
using models of human behavior to choose machine policies optimizing interaction `rewards,'
i.e., the effectiveness of the sequence of machine actions
in nudging human towards desirable outcomes. To that end, it is desirable to
choose machine policies which are not overly sensitive to sampling
noise in data collection and/or variability in human behavior.
Also, of interest are policies that are not overly deterministic/predictable, as
in some settings, such policies may be poor in keeping humans `engaged',
see, e.g., \cite{EpN15}, and poor
in eliciting rich human-machine data sets.

%(1) Even human-machine interactions can be repeated by the same person
%or over a large population, humans' behavior may change over time. For
%interactions where outcomes depend subtly on the history one can expect
%the exponential complexity to be a challenge.
%(2) To infer a model for human behavior one must choose a machine policy
%to ``interact'' with humans, which may in turn lead to ``biased'' inferences.
%Also, inferences should respect the fact that machines and humans can only make decisions
%based on the information that has been revealed to date,
%so the causal nature of such interactions must
%be considered.
%(3) Thus the optimization of such interactions involves two concerns.
%One one hand, the machine policy should engage humans to collect data
%for human behavior inferences. Also, it should nudge human towards desirable
%outcomes, and be robust to sampling noise and variability in human
%behavior.

\emph{Contributions}.
In this paper, we propose a data-driven framework to jointly solve the estimation and
optimization problems associated with human-machine interaction processes.
We adopt an inference technique
based on a constrained maximum entropy principle for interactive processes, see \cite{ZBD11, ZBD13}.
This allows one to incorporate prior knowledge
of the (possibly) relevant features of human behavior,
via moment constraints associated with interaction functions.
We consider optimizing machine policies based on an
interaction  reward function with an entropy-based regularization term.
This aims to find machine policies which
maximize rewards, are robust to estimation noise, and
maintain a degree of exploration when interacting with humans.
Our proposed Alternating Reward-Entropy Ascent (AREA) Algorithm,
alternates between data-collecting, estimation of human behavior, and the
optimization of machine policy, with a view on reaching a consistent fixed point.
We provide a characterization of various properties
of AREA.  In particular, for decomposable and/or path-based feature and
reward functions, we devise a computationally efficient
approach to estimation and optimization steps. The approach
takes advantage of defining a stopping time over the interaction and
the conditional Markov property of the estimated human model,
to significantly reduce space and time complexity.
We provide a theoretical characterization of the AREA algorithm
in terms of its convergence, along with simple preliminary
evaluation results based on synthetic data obtained
from a noisy nonlinear model for human decision-making.
\ifappendix
All proofs are included in the appendix.\fi

%{\color{red}\textbf{Jiaxiao: guess we need to remove the paper organization.}}
%\emph{Paper organization}. In \cref{sec:formation},
%we introduce a model for human-machine interaction processes, and formulate the associated estimation
%and optimization problems and iterative AREA algorithm.
%In \cref{sec:solution}, we provide the solution to
%both these problems, and in \cref{sec:complexity},
%we show that under additional assumptions the complexity of AREA iterations
%is polynomial in space and time. In \cref{sec:area}, we
%characterize some convergence properties for AREA. \cref{sec:evaluation},
%provides a representative numerical evaluation of AREA based on synthetic
%human-machine interaction data. All the proofs,
%together with an interesting special case of AREA, can be found in appendicies.

%% file: formulation.tex
\section{Problem Formation}\label{sec:formation}

We shall consider a structured
discrete time human-machine \emph{interaction process} over a period of time $1,2,\dots, T$,
which can be
viewed as a pair of sequences of random variables, $(H_1, \dots, H_T)$ corresponding to human
actions/responses if any,  and $(M_1, \dots, M_T)$ denoting those of the machine.
We shall assume the random
variables $H_t, M_t$ capture discrete human and machine actions at time $t$, and,
without loss of generality, that for all $t$, $H_t\in \mathcal{H}$, where $\mathcal{H}$ denotes the human's
action space, and $M_t\in \mathcal{M}$, where $\mathcal{M}$ corresponds to
the machine's action space\footnote{Note both the
human and/or machine could choose to do nothing in their turn. This can be included in our model
by including null action in both $\mathcal{M}$ and $\mathcal{H}$.}.
Throughout this paper we assume that $|\mathcal{H}|$ and $|\mathcal{M}|$
are finite.
To simplify notation, we let $H^t = (H_1,\ldots, H_t)$ for $t=1,\ldots, T$,
and similarly define $M^t$. When $t = 0$, $H^t$ or $M^t$ contains no elements. 
We assume that the human and the machine take turns,
such that the machine's action at time $t+1$, i.e., $M_{t+1}$ depends on
$H^t,M^t$ while that of the human at time $t+1$, i.e.,
$H_{t+1}$ depends on $H^{t},M^{t+1}$.
The joint distribution of $(H^T, M^T)$ captures the
interplay between the human and machine. Depending on the
setting, the human refers to a particular individual or
a population, where the behavior
can be captured via a stable distributional model.

We shall assume that when a human and machine interact, the
machine's policy is known and captured by a collection
of conditional distributions $Q$, for succinctness denoted by
$
Q(m_t|h^{t-1},m^{t-1}) :=  p_{M_{t}|H^{t-1},M^{t - 1}}(m_{t}|h^{t-1},m^{t-1}) ~~\mbox{for}~~ t = 1,\ldots T.
$
Similarly human behavior is denoted by conditional distributions $P$ given by
$
P(h_t|h^{t-1},m^{t}) :=  p_{H_{t}|H^{t-1},M^{t}}(h_{t}|h^{t-1},m^{t}) ~~\mbox{for}~~ t = 1,\ldots T.
$
It is easy to show that joint distribution of $(H^T, M^T)$,
denoted by $PQ$, resulting from a human model $P$ interacting with
a machine policy $Q$, can be decomposed as
$
PQ(h^T,m^T) = P(h^T || m^T) Q(m^T || h^{T}),
$
where
\begin{equation}\label{eq:human-factor}
P(h^T || m^T) := \prod_{t=1}^{T} P(h_t|h^{t-1},m^t)
~~~~\mbox{and}~~~~
Q(m^T || h^T) := \prod_{t=1}^{T} Q(m_t|h^{t-1},m^{t-1}),
\end{equation}
correspond to the \emph{causally conditioned distributions} of the human and the machine,
i.e., products of sequentially conditioned distributions.
We will assume that data of human-machine interactions can be collected by fixing a machine policy,
and keeping track of the realizations of such interactions.

We let $P^*(h^T\| m^T)$ denote the {\em true} human behavior
and $\hat{P}(h^T\| m^T)$ an estimated model thereof.
We let $PQ(A)$ denote the probability of an
event $A$ measurable w.r.t. $(H^T, M^T)$ and
we let $E_{PQ}[f(H^T, M^T)]$ denote the expectation of a function
$f(h^T, m^T): {\cal H}^T \times {\cal M}^T \rightarrow \mathbb{R}$
under the joint distribution $PQ$.
When we collect interaction data of the human with a machine policy
$Q(m^T\|h^T)$ we denote expected value under the associated empirical distribution by
$\hat{E}_{P^*Q}$ where in the ideal case (no noise)
we have $\hat{E}_{P^* Q}[{f}(H^T, M^T)]= E_{P^*Q}[{f}(H^T, M^T)]$.
\ifappendix
Those notations are summarized in \cref{tab:notations}.
\fi

%This decomposition based on causally conditional distributions should be contrasted with the more traditional decomposition:
%$$
%p_{H^T, M^T}(h^T,m^T) = p_{M^T|H^T}(m^T | h^T) p_{H^T}(h^T),
%$$
%where
%$$
%p_{M^T|H^T}(m^T | h^T) = \prod_{t=1}^{T} p_{M_t | M^{t-1}, H^T}(m_t|m^{t-1},h^T),
%$$
%which is less useful for our purposes, as it does not expose the causally conditioned actions of
%the human(s) and machine policy.
%Indeed the second decomposition is less natural in the context of interaction processes since
%the distribution for the machine's decision is conditioned on
%future human responses, which are not available yet.

\ifappendix
\begin{table}[]
\centering
\caption{Table of notations}
\label{tab:notations}
\begin{tabular}{|l|l|}
\hline
Sequence of human actions & $H^t = \{H_1, H_2, \cdots, H_t\},~0\le t\le T$ \\\hline
Specific realization of human action  & \multirow{2}{*}{$h^t = \{h_1, h_2, \cdots, h_t\},~0\le t\le T$}\\
sequence & \\\hline
Sequence of machine actions & $M^t = \{M_1, M_2, \cdots, M_t\},~0\le t\le T$ \\\hline
Specific realization of machine action  & \multirow{2}{*}{$m^t = \{m_1, m_2, \cdots, m_t\},~0\le t\le T$}  \\
sequence & \\\hline
Joint PMF of $M^T, H^T$ & $p_{H^T, M^T}(h^T, m^T)$   \\\hline
Causally conditional distribution of  & \multirow{2}{*}{$p_{H^T\|M^T}(h^T\|m^T)$ or $P(h^T\|m^T)$}\\
human actions given machine actions &  \\\hline
Causally conditional distribution of  & \multirow{2}{*}{$p_{M^T\|H^T}(m^T\|h^T)$ or $Q(m^T\|h^T)$}\\
machine actions given human actions &  \\\hline
Joint PMF when the human model is  & \multirow{3}{*}{$PQ(h^T, m^T)$}\\
$P(h^T\|m^T)$ and machine model is  &  \\
$Q(m^T\|h^T)$ & \\\hline
Probability of event $A$ when the human & \multirow{3}{*}{$PQ(A)$}\\
model is $P(h^T\|m^T)$ and machine model & \\
 is $Q(m^T\|h^T)$. &\\\hline
Expectation of function of interactions  & \multirow{3}{*}{$\mathbb{E}_{PQ}[f(H^T, M^T)]$} \\
$f(H^T, M^T)$  w.r.t. the joint PMF   & \\
given by $PQ(h^T, m^T)$ & \\ \hline
The actual human behavior model & $P^*(h^T\|m^T)$ or $P^*$ \\ \hline
The estimated human behavior model & \multirow{2}{*}{$\hat{P}(h^T\|m^T) = h(Q)$ or $\hat{P} = h(Q)$} \\
if the machine model is $Q$ & \\ \hline
The machine model if the estimated  & \multirow{2}{*}{$\hat{Q}(m^T\|h^T) = m(P)$ or $\hat{Q} = m(P)$} \\
human model is $P$ & \\ \hline
\end{tabular}
\end{table}
\fi

\subsection{Data-driven human model estimation}

A brute force approach to modeling human behaviour
would be to directly estimate the conditional
distributions $\{P^*(h_t | h^{t - 1}, m^t), t =1,\ldots,T \}$
based on the collected data which is clearly not scalable.
Instead, in this paper we embrace the extension of
constrained maximum entropy estimation to interactive processes developed in~\cite{ZBD13,WZB13}.

In this setting, one defines a set of feature functions
ideally known to capture relevant characteristics of human behavior
which become equality and inequality constraints in the estimation process.
{The choice of such features would be motivated by known
frameworks for understanding {\em human behavior} in dynamic environments, e.g,
the effort accuracy \cite{PBJ93}, exploration-exploitation \cite{MaM10},
soft constraints \cite{GrF04}, and specific character of the human-machine interaction.}
The equality constraints are based on matching
the moments of a set of feature functions
$\mathcal{F}$ denoted by $\mathbf{f}(h^T, m^T):=\{f^i(h^T, m^T),~i \in \mathcal{F}\}$,
and their moments based the empirical distribution
when interacting with a given machine policy $Q$, which are
denoted by $\mathbf{c}_f := \hat{E}_{P^*Q}[\mathbf{f}(H^T,M^T)]$.
Below we will neglect sampling errors by assuming that 
$\hat{E}_{P^*Q}[\mathbf{f}(H^T, M^T)] = E_{P^*Q}[\mathbf{f}(H^T, M^T)].$
The set of inequality constraints are denoted by
$\mathbf{g}(h^T, m^T):=\{g^i(h^T, m^T),~i \in \mathcal{G}\}$,
where $\mathcal{G}$ is another set of
feature functions
whose moments are constrained not to exceed pre-specified thresholds
$\mathbf{c}_g = \{c_g^i,~i \in \mathcal{G}\}$.

Formally, for a given machine policy $Q(m^T\|h^T)$, we are interested in models for
human behaviour $P(h^T\|m^T)$  satisfying the following constraints
%such that $E_{PQ}[\mathbf{f}(H^T, M^T)] = \mathbf{c}_f = \hat{E}[\mathbf{f}(H^T,M^T)]$
%and $E_{PQ}[\mathbf{g}(H^T, M^T)] \ge \mathbf{c}_g$.
%
%The feature moment constraints given a machine behavior $Q(m^T\|h^T)$ are expressed as
%(In this paper we won't consider the random sampling error):
\begin{align}
\label{eq:moment-matching}
{\mathcal{P}_{\mathcal{F},\mathcal{G}}^Q} = & \{P(h^T\|m^T)~|~
E_{PQ}[\mathbf{f}(H^T,M^T)] = \mathbf{c}_f,
\textrm{and}~E_{PQ}[\mathbf{g}(H^T,M^T)] \ge \mathbf{c}_g\}.
\end{align}

The maximum entropy estimation principle chooses the model for human behaviour in
${\mathcal{P}_{\mathcal{F},\mathcal{G}}^Q}$ with maximum entropy. In the
case of interactive processes, since the machine policy $Q(m^T\|h^T)$
is known we shall maximize the entropy of the causally conditioned
distributions of the human behavior model.
In particular, the causally conditioned entropy of human behaviour model
${P}(h^T \| m^T)$ given machine policy in use is $Q(m^T \| h^T)$, is given by
\begin{equation}
\label{eq:causalentropy}
\mathbb{H}_{{P}Q}(H^T \| M^T)
:= E_{{P}Q}\left[-\log\left({P}(H^T \| M^T)\right)\right] = \sum_{t = 1}^T\mathbb{H}_{{P}Q}(H_t|H^{t-1}, M^t).
\end{equation}

In the sequel we consider optimizing functionals
of the causally conditioned distributions for the human (and the machine).
Doing so means optimizing over a set of
conditioned distributions $\{P(h_{t}|h^{t-1},m^{t}) \, | \, t=1,\ldots,T \}$,
which for simplicity we {\em also} denote by $P(h^T\|m^T).$
It can be shown that these collections of distributions belong
to a convex polytope denoted by ${\cal C}_H$ (resp. ${\cal C}_M$).
Indeed, according to \cite{ZBD11}, $P(h^T\| m^T)\in {\cal C}_H$ is equivalent to
the requirement that $P(h^T\| m^T)$ can be factorized into a product of conditional distributions as in (\ref{eq:human-factor}).
Similar result holds true for $Q(m^T\|h^T)$.
This generalizes the notion of optimizing over
a set of distributions with a given support, e.g.,
over the simplex. In the sequel for the sake of simplicity, we will
omit the constraints $P(h^T\|m^T) \in {\cal C}_H$ and $Q(m^T\|h^T)\in {\cal C}_M$
when they appear in optimization problems--it is assumed to be understood that one is optimizing
over causally conditioned distributions that must be properly normalized.
The overall human estimation problem can thus  be expressed as follows.

\begin{definition}\textbf{(Human estimation problem)}
Given a known machine policy $Q(m^T\|h^T)$ and a set of
moments $\mathbf{c}_f$ associated with human-machine interaction for equality constraints,
the constrained maximum entropy estimate model for human behavior, say
$\hat{P}(h^T \| m^T)$ is the solution to the following problem:
\begin{equation}\label{eqn:step1}
\max\limits_{P(h^T \| m^T)}
\{~ \mathbb{H}_{PQ} (H^T\|M^T)~ |~
P(h^T\|m^T) \in \mathcal{P}^Q_{\mathcal{F}, \mathcal{G}} ~\}.
\end{equation}
%\begin{aligned}
%& \max\limits_{P(h^T \| m^T)}
%&&
%\mathbb{H}_{PQ} (H^T\|M^T)
%\\
%& \mbox{\textrm{s.t.}}
%&&
%P(h^T\|m^T) \in \mathcal{P}^Q_{\mathcal{F}, \mathcal{G}}.
%\end{aligned}
\end{definition}
Note that since this problem is convex, the solution $\hat{P}(h^T \| m^T)$ is unique. However,
it depends on underlying machine policy $Q$ both through the cost function and the constraints.

%Let $\{\hat{P}_Q(h^T\|m^T),\forall h^T,m^T\} = h(Q)$ denote the solution to this problem, to indicate that
%the estimate of human model is dependent on the underlying machine's policy.
%In the sequel, this is written in short as
%$\hat{P}_Q = h(Q)$ when there is no ambiguity.

\subsection{Machine optimization}

%Our objective is to optimize a reward function defined over such human-machine interactions
%by tuning the machine policy.
%However the conditional distributions capturing human behavior are complex and typically unknown.
%Thus we do not know the relationship between the joint and the machine's policy.
%In our proposed approach, we need to first obtain a model for human behavior and
%then optimize the reward by tuning the machine's policy, assuming
%our estimates are accurate enough.
%We thus decompose the optimization problem into two subproblems.

%\begin{enumerate}
%\item The definition and notation of entropy is mostly the same as that for human model estimation.
%\item Formation of the optimization problem. Motivate the entropy.
%\end{enumerate}

We assume one has defined a \emph{reward function} $r(h^T, m^T)$ over human-machine interactions.
This function might reflect both desirable human outcomes/decisions as well as machine costs for taking
certain sequences of actions.
Given an estimated  model for the human behaviour, $\hat{P}(h^T\|m^T)$,
one can in turn consider choosing a reward maximizing machine policy, i.e.,
$$
\max_{Q(m^T\|h^T)}~~E_{\hat{P}Q}[r(H^T, M^T)].
$$
A direct optimization of the reward as above would result in machine policies
that take deterministic actions associated with the `best' choices. Such policies
are likely to be vulnerable to 
the error in the estimated human behaviour model due to the sampling noise.
This has also been observed in the context of reinforcement learning,
see, e.g., \cite{AS97,S1996}.
Such machine policies may also be limited in the degree to which `explore' interaction with the human,
and thus subsequently the obtained interaction data may lead
to poor estimates of human behavior and sub-optimal results.
Further, we also posit that
deterministic machine policies have poor characteristics from a human interaction
perspective, e.g., might also be boring/too predictable,
leading to poor engagement \cite{EpN15},
and/or in certain settings may be unfair.
For example, in an advertising setting, one might want to incorporate randomness in
placing advertisements to ensure fairness and/or encourage competition.

To address these concerns we propose adding a `regularizing' entropy term
to the reward function.
Thus given an estimated model for human behavior $\hat{P}$, the machine's policy is obtained
as the solution to the following problem.

\begin{definition}\textbf{(Machine policy optimization problem)}
Given an estimated model for human behavior $\hat{P}(h^T\|m^T)$,
the reward maximizing machine policy is given by the solution to
\begin{equation}\label{eqn:step2}
\begin{aligned}
& \max\limits_{Q(m^T \| h^T)}
&&
\mathbb{H}_{\hat{P}Q} (M^T\| H^T)  + \gamma  E_{\hat{P}Q}[r(H^T,M^T)], \\
\end{aligned}
\end{equation}
\end{definition}
where $\gamma > 0$ controls the degree to which one weighs entropy versus reward in
the machine policies.
We shall realize that this formulation is in fact similar to human estimation problem introduced earlier.

\subsection{Closing the loop: Alternating Reward-Entropy Ascent (AREA) Algorithm}
\label{sec:closing-the-loop}
Note that the optimized machine policy obtained via (\ref{eqn:step2}) depends on a
estimated model for human behavior, which in turn was estimated by solving (\ref{eqn:step1}) based
on data obtained from human-machine interactions using the previously selected machine policy.
The two machine policies need not to be the same, possibly making the estimation and optimization steps
inconsistent.
To resolve this, we propose \emph{Alternating Reward-Entropy Ascent}
(AREA) algorithm exhibited in \Cref{overview-fig}. We begin with
a default machine policy (for example, the machine might choose actions at random),
denoted by $\hat{Q}^{(0)}(m^T\|h^T)$. Under this machine policy we collect data/realizations
of human machine interactions.  Then from the data, we can estimate the feature moments,
which, in turn, enable estimation of a model for human behavior
$\hat{P}^{(0)}$ through our \emph{inference} phase, i.e.,
(\ref{eqn:step1}).  Based on the estimated model of human behavior we generate a new machine policy
through the machine \emph{optimization} phase,
where the optimization is based on $\hat{P}^{(0)}$,
obtaining the next machine policy $\hat{Q}^{(1)}$.
This alternating process generates a sequence of causally conditioned distributions given by
$\hat{Q}^{(0)} \rightarrow \hat{P}^{(0)} \rightarrow \hat{Q}^{(1)} \rightarrow \hat{P}^{(1)} \rightarrow \dots$, which we refer to as AREA iterations.
%In our work we assume that we use the same set of feature and constraint set $\mathcal{F}$
%and $\mathcal{G}$
%in all human estimation steps and use the same reward function $r(h^T, m^T)$
%in all machine optimization steps.
%Note that the inclusion of entropy term when optimizing machine policies
%ensures a degree of randomness for collecting data on human-machine
%interactions throughout this iterative process.

\begin{figure}[]
\center
\includegraphics[width=0.7\textwidth]{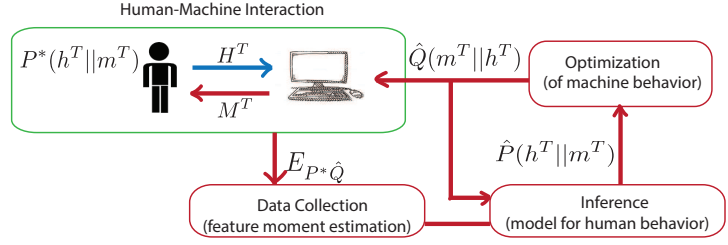}
\caption{Overview of framework for the optimization of human-machine intearctions.}
\label{overview-fig}
\vspace{-1.5em}
\end{figure}

%% file: related-work-simplified.tex
\section{Related Work}
{ {\em Markov decision processes and reinforcement learning:}}
The optimization of human-machine interactions can in principle be modelled
as a Markov Decision Process (MDP), where the human behavior
can be viewed as driven by a
transition kernel among a set of states,
and the machine behavior corresponds to a sequence
of actions taken in response to the human's behavior. The underlying assumptions are that there
exists a state space for the human and an action space such
that the distribution of future states depends only on the current state (say of the human)
and chosen action (say of the machine). In such a setting, one can define a reward function and
consider optimizing the associated machine policy, see e.g., \cite{Put94,Ber00,BaR11}.
When the transition kernel is unknown,
but assumed Markov, the resulting problem is known as reinforcement learning, see e.g., surveys~\cite{KLM96, SB98}.
Both model-based and model-free reinforcement learning approaches
(and methods that combine both approaches) have been studied in the literature.
Model based methods combine estimation of the environment and optimization
of machine actions, while model-free methods aim to directly optimize the machine
without first estimating a model the environment. For example,
Q-learning aims to directly estimate the value of state-action pair, denoted by $Q(s, a)$, where the $s$ is the current state
and $a$ is the candidate action.
The $Q$-function can be used to select the optimal sequence of machine actions~\cite{KLM96, SB98}.
The traditional framework of reinforcement learning relies heavily on the
assumption that the underlying environment is Markov. However as deep learning technologies
emerge, deep Q-learning \cite{VKD15} have been devised to approximately solve
this problem. Indeed to overcome the difficulties brought by non-Markov environments,
an option is to first
enlarge the state space of the underlying environment significantly,
for example, to include all possible history of the system.
Then use a deep neural network to encode the $Q$-function, and
fit the neural network to the observed data. This approach also has its challenges
in terms of demanding data requirements and might not be applicable to some use cases.

In our framework, when the reward function is decomposable over time
i.e. $r(h^T, m^T) = \sum_{t=1}^T r(h_t, m_t)$, and the estimated human
model $\hat{P}(h^T || m^{T})$ is one-step Markov, the machine optimization program reduces to a
traditional MDP setting, with a possible time-inhomogeneous transition kernel
and the reward function is regularized by an entropy term to promote
exploring different actions.Some recent literature suggests that model based methods may be preferred to model free methods in terms of sample complexity~\cite{AS97,S1996}.
In the special case of a Markov model, our approach may be considered as a variation of model-based reinforcement learning, where the model is learned by maximizing causal entropy subject to moment constraints, and the machine behavior is regularized using the causal entropy of the machine process. As discussed, this analogy no longer holds for the general case.

We are aware of only a few cases where (relative) entropy regularization has been combined with Markov decision processes and related models. \cite{GRW14} consider a generalization of the Markov decision process where,
%, relative entropy was used in the cost function in order to model the control cost.
instead of impacting the process through some actions, the agent can directly manipulate the transition matrix of the system state. However, such manipulation would incur some cost which is proportional to the relative entropy between the transition probability after manipulation, and the transition matrix of a `passive' process which models the `natural' system evolution.
In \cite{MeR05}, the authors propose an entropy-regularized cost function to approximately solve a partially observable Markov decision process (POMDP) model efficiently. Due to the absence of knowledge of the exact system state (i.e. partial observation), the agent must estimate it through the reward it receives and a noisy observation of its current state. Therefore, there is a trade-off between gaining more profit based on current belief -- which requires focusing on the most profitable action, and improving the quality of estimation -- which requires exploring different actions. The authors of \cite{MeR05} used the expectation of entropy in the agent's belief state as a proxy of how well it explored different actions. %If the entropy in the agent's belief is expected to decrease a lot after an action, the action is believed to perform a good exploration.
The main
challenge associated with MDP is that the human's behavior transition kernels may have long term
dependencies -- and an extremely large state space may be required state to remain in the Markov setting.

\emph{Bandit problem: }
The state-of-the-art approach to solving the problems with such sequential and interactive context
also includes multi-armed bandit problem and its variants \cite{BuC12}, which are widely discussed and used in
use cases including computational advertising.
In such context, the search
engine uses the user feature including gender, age and searching history as the context,
to pick up an ad, which is modeled as the arm, after each user's query,
such that the user will have a good chance
of clicking through the ad.
The most representative method is the ILOVETOCONBANDITS algorithm proposed
in \cite{AHK14}, where it is assumed that the reward received for each
attempt depends on some observable random `context'.
However the approach depends heavily on the i.i.d. assumption on the
environment, in order to improve the quality of the estimation by accumulating samples.
Therefore when
the user does not make independent decisions or has a long-term memory, the performance of such contextual bandit based
solutions will not be acceptable.
%Even if the human does indeed make independent decisions, contextual bandit based solutions eventually converge to a repetitive interaction pattern for the machine's policy, which might, in turn, undermine the independence assumptions, as documented in the psychology literature, see e.g. \cite{EpN15}, humans tend to view repetitive or deterministic processes as boring and stop paying attention or alternatives with some degree of randomness may be intriguing.

The most general way to model such problems in a multi-armed bandit way is continuum armed bandit, for example,
\cite{TyG13}, where the arms to pick can be a vector of real numbers instead of discrete index.
We can directly model the machine's policy $Q(m^T\|h^T)$ as arms. However when the support of the arm is big, the convergence of the algorithm
is slow, and also it requires a prior knowledge of the number of iterations we need, thus cannot be
implemented in a fully online manner.

%% file: solution.tex
\section{Solution to AREA's Optimization Problems}\label{sec:solution}

The Lagrangians for the optimization problems (\ref{eqn:step1}) and (\ref{eqn:step2}) have similar forms.
We shall begin our discussion of the solution approach, based on \cite{ZBD13},
for the human estimation problem and subsequently that of the machine optimization,
pointing out some key results and notation that will be critical for our development in the sequel.

\subsection{Solution to human estimation problem}
\label{sec:solution-estimation}

It has been shown in \cite{ZBD13}
that the human estimation problem
is concave in $P(h^T\|m^T)$ given $Q(m^T\|h^T)$, and the solution can be found by its dual.
\ifappendix
\begin{theorem}\cite{ZBD13} \label{thm:solution}
The dual form of the human estimation problem (\ref{eqn:step1}) is given by:
\begin{align}\label{eq:estimate-dual}
\min_{\substack{\bm{\lambda} = (\bm{\lambda_f}, \bm{\lambda_g}),\\ \bm{\lambda}_g \le 0}}
\sum_{m_1}Q(m_1)\log Z_{\bm{\lambda}}(m_1)
 -\bm{\lambda}_f^T\mathbf{c}_f -\bm{\lambda}_g^T\mathbf{c}_g
\end{align}
where
\begin{equation}
Z_{\bm{\lambda}}({h}^t, {m}^{t+1}) = \sum_{h_{t+1}}Z_{\bm\lambda}(h_{t+1}|{h}^t, {m}^{t+1}),~~Z_{\bm\lambda}(m_1) = \sum_{h_1}Z_{\bm\lambda}(h_1|m_1)
\end{equation}
and
\begin{align}
Z_{\bm\lambda}(h_t|h^{t-1}, m^{t})
 =  \left\{ \begin{array}{ll}
e^{\sum_{m_{t+1}}Q(m_{t+1}|h^t, m^t)\log Z_{\bm\lambda}(h^t, m^{t+1})} & t < T \\
e^{\bm{\lambda}_f^T \mathbf{f}_1(h^{T}, m^T) + \bm{\lambda}_g^T \mathbf{f}_2(h^T, m^T)} & t = T
\end{array}
\right.
,
\label{eq:backward}\end{align}
The associated human model for dual variables
$\boldsymbol{\lambda}$ is given by
$
{P}_{\bm\lambda}(h_t|h^{t-1}, m^t) = \frac{Z_{\bm\lambda}(h_t|h^{t-1}, m^{t})}{Z_{\bm\lambda}(h^{t-1}, m^{t})}.
$
\end{theorem}

The optimal dual $\bm\lambda^*$ can be found by subgradient-based method, see \cite{ZBD13}
or \Cref{sec:solution-estimation-dual}.
\else
\fi
In the sequel it will be useful to denote the solution to the human estimation problem
by $h^*(Q, \mathbf{c}_f, \mathbf{c}_g)$ to make clear its dependence on $Q$
the machine policy,  $\mathbf{c}_f$ the feature moments estimated from human-machine interactions,
and the constants $\mathbf{c}_g.$

\ifappendix
The solution given in \cref{thm:solution} has several interpretations, two of which are given in following two theorems.
\begin{theorem}\label{thm:min-logloss}
\cite{ZBD13} Using statistics from the true distribution without sampling error, maximizing the causal entropy subject to feature constraints in human estimation problem is equivalent to maximizing the log causal likelihood of the true distribution over the family of \emph{causal Gibbs distributions}.
\begin{equation}
\max_{\bm\lambda} E_{P^*Q}[\log P_{\bm\lambda}(H^T\|M^T)]
\end{equation}
\end{theorem}

\begin{theorem}
\cite{ZBD13} The human estimation problem is equivalent to minimizing the \emph{worst case causal log-loss} when the true human behavior is chosen adversarially.
\begin{eqnarray}
\inf_{P(h^T\|m^T)}\sup_{P^*(h^T\|m^T)} && E_{P^*Q}[-\log P(H^T\|M^T)] \nonumber\\
\textrm{s.t. } && E_{PQ}[\mathbf{f}(H^T, M^T)]  = E_{P^*Q}[\mathbf{f}(H^T, M^T)]
\end{eqnarray}
\end{theorem}
\fi

\subsection{Solution to machine optimization problem}

It should be clear at this point that the
the objective function in (\ref{eqn:step2})
is similar to the Lagrangian of Problem (\ref{eqn:step1})
with a fixed `dual variable' $\gamma$.
Thus the following result is fairly straightforward.

\begin{theorem}\label{thm:solution-machine}
For a given model of human behavior $\hat{P}(h^T\|m^T)$
the solution to the machine optimization problem (\ref{eqn:step2}),
$\hat{Q}(m^T\|h^T)$ is given as follows.
Let
\begin{align}
Y_{\gamma}(m_t|h^{t-1}, m^{t-1})
 =  \left\{ \begin{array}{ll}
e^{\sum_{h_t}\hat{P}(h_t|h^{t-1}, m^t)\log Y_{\gamma}(h^t, m^t)} & t < T \\
e^{\gamma \sum_{h_T}\hat{P}(h_T | h^{T-1}, m^T)r(h^T, m^T)} & t = T
\end{array}
\right.,
\label{eq:backward-machine}\end{align}
where
$Y_\gamma(h^t, m^t) = \sum_{m_{t + 1}}Y_\gamma (m_{t+1} | h^{t}, m^t),~~Y_\gamma = \sum_{m_1}Y_\gamma(m_1)$.
Then the optimal machine policy is
$
\hat{Q}(m_t|h^{t-1}, m^{t-1}) = \frac{Y_{\gamma}(m_t|h^{t-1}, m^{t-1})}{Y_\gamma(h^{t-1}, m^{t-1})}
$
and
$
\hat{Q}(m_1) = \frac{Y_\gamma(m_1)}{Y_\gamma}.
$
\end{theorem}

\ifappendix
Please see \cref{sec:proof-thm:solution-machine} for detailed proof.
\fi

In the sequel it will be useful to represent the result stated in \Cref{thm:solution-machine} as follows.
In particular the auxiliary function
$\mathbf{Y}_\gamma := \{Y_\gamma(m_t|h^{t-1}, m^{t-1}), \forall 1\le t \le T\}$
generated by the procedure given in \Cref{thm:solution-machine}
depends on the human model and  so is denoted by $\mathbf{Y}_\gamma = m(\hat{P}).$
The associated optimal machine policy  $\hat{Q}$ is in turn
a function of $\mathbf{Y}_\gamma$ denoted by $\hat{Q} = m^*(\mathbf{Y}_\gamma).$

%% file: complexity.tex
\section{Complexity of AREA Algorithm}\label{sec:complexity}
As can be seen, the dual problem of human estimation problem
is over a vector $\boldsymbol{\lambda}$ of dimension $|\mathcal{F}| +|\mathcal{G}|$.
The authors of \cite{ZBD13} shows that we can find the optimal dual variables by a recursion
only involves computing the expectation of feature functions,
respect to joint distribution $P_{\boldsymbol{\lambda}}Q$,
where $P_{\boldsymbol{\lambda}}$ is the human distributional model
associated with $\boldsymbol{\lambda}$.
However when updating the dual variables, computing those expectations
are intractable
in the most general setting.
Specifically, if we define the space complexity as the number of variables that need to be stored,
and the time complexity
as the number of basic math operations (e.g. addition, multiplication and exponential function evaluation)
required to carry out the update,
we can see that because the number of conditioning sequences in
\ifappendix(\ref{eq:backward})\else the solution to the human estimation problem \fi
grows exponentially in $T$, thus if we need to put all conditional PMFs into the memory
and then compute the joint PMF accordingly, both space and time complexities required
are exponential in $T$.
Fortunately, when the feature functions have specific forms, the
complexity of computing such updates can be reduced. Specifically,
we will discuss cases where
one iteration of AREA algorithm described in \Cref{sec:closing-the-loop}
has polynomial complexity in $T$.

\begin{definition}
A feature function $f(h^T, m^T)$ is said to be decomposable if it can be written as
$
f(h^T, m^T) = \sum_{t = 1}^T f_t(h_t,m_t).
$
\end{definition}

\begin{definition}
A function $f(h^T, m^T)$ is said to be path-based if it is proportional to the indicator function
of a specific realization of the human-machine interaction, say $(\bar{h}^T, \bar{m}^T)$, i.e.,
$
f(h^T, m^T) = c \mathbf{1}_{\{ (h^T, m^T) = (\bar{h}^T, \bar{m}^T) \}}.
$
%where different path-based features can have different $\bar{h}^T, \bar{m}^T$,
%We refer to $(\bar{h}^T, \bar{m}^T)$ as the support of $f(h^T, m^T)$.
\end{definition}
%Also note that if $c\ne 1$, in problem (\ref{eqn:step1}),
%the dual update in \Cref{eq:dual-update} can be regarded as tuning $\lambda c$
%as a whole such that the induced expectation of feature functions match the feature moments
%obtained from the data. Thus changing $c$ will not affect
%the estimated distributional model. In the sequel, when isolated path based features are introduced,
%i.e., not sums thereof, we will assume $c = 1.$

Note that it is always desirable to include the
reward function in the equality feature set ${\cal F}$ to
ensure that the estimated human model matches
the true human behavior in terms of the associated
mean rewards. Then we have the following result.

\begin{theorem}\label{thm:complexity}
Suppose the reward function $r(h^T,m^T)$ can be written as a sum of
a decomposable function and a set $\mathcal{R}_p$ of path-based functions,
and the remaining feature functions are either decomposable or path-based,
i.e., $\mathcal{F} = \mathcal{F}_p \cup \mathcal{F}_d \cup \{r(h^T, m^T)\}$
and $\mathcal{G} = \mathcal{G}_p \cup \mathcal{G}_d$,
where $\mathcal{F}_p$ and $\mathcal{G}_p$ denote path-based
equality/inequality features, and
$\mathcal{F}_d$ and $\mathcal{G}_d$ decomposable
equality/inequality features, respectively.
Suppose further that the initial machine's policy $\hat{Q}^{(0)}$ is
uniformly random.
Then the space complexity of each dual update of human estimation problem is
$O\left((|\mathcal{F}_{p}| + |\mathcal{G}_{p}| + |\mathcal{R}_p|)T|\mathcal{H}||\mathcal{M}|\right)$,
and the time complexity of each dual update is $O(T(|\mathcal{F}_{p}| + |\mathcal{G}_{p}| + |\mathcal{R}_p|) \max(T, |\mathcal{H}||\mathcal{M}|))$.
The time and space complexity of the machine optimization problem are both
$O((|\mathcal{F}_{p}| + |\mathcal{G}_{p}| + |\mathcal{R}_p|)T|\mathcal{H}||\mathcal{M}|)$.
%Thus the complexity of each iteration of AREA algorithm is polynomial in $T$, if the
%number of dual updates in human estimation problem is limited.
%{\bf XXX wording in this last statement a bit awakward.}
\end{theorem}

\ifappendix
Please see \cref{sec:proof-thm:complexity} for detailed proof.
\fi

\emph{Remark:} We envisage that the inclusion of path-based and decomposable feature
and reward functions might allows a fairly rich framework to capture relevant
interaction characteristics. In particular path-based features are capable of modeling detailed
long-term memory in human-machine interactions
while decomposable features can model short-term dependencies.
As shown in \Cref{thm:complexity}, for such settings,
the solution to (\ref{eqn:step1}) and (\ref{eqn:step2}) require steps
with only polynomial space and time complexity.
\ifappendix \else Please see \cite{preprint} for detailed proof.\fi

%\emph{Remark}: Similar result holds true for the machine optimization step. Therefore, if only pathwise and decomposable features
%are adopted in the framework, the AREA procedure will have only polynomial time and space complexity in each step.

%% file: AREA.tex
\section{AREA Convergence}\label{sec:area}
As discussed in previous sections, the AREA Algorithm is aimed at
achieving high rewards through consistency in the estimated human model
and optimized machine policy.
In this section, we characterize AREA's convergence properties.

The convergence of the algorithm can be guaranteed in two extremal cases.
Clearly if the set of feature functions
$\mathcal{F}$ and $\mathcal{G}$ is rich enough
that the true human behaviour is recovered as the solution to (\ref{eqn:step1}),
then AREA converges\ifappendix--see \Cref{sec:sufficient-stats} for details\fi.
Or, if the feature set is sufficient to guarantee that the actual human behavior
along the `paths' that are impactive to reward is perfectly captured by the estimated
human model, then AREA also converges in one iteration.
\ifappendix
\begin{theorem}\label{thm:path-reward}
Suppose the reward function is a path-based, i.e.,
$r(h^T, m^T) = \mathbf{1}_{\{(h^T,m^T )=(\bar{h}^{r, T}, \bar{m}^{r, T})\}}$
and included in the feature set
$\mathcal{F}$ and the initial machine's
policy $\hat{Q}^{(0)}$ has full support.
Consider a modified version of human estimation problem
which includes the following additional features.
For each path-based feature, i.e., $i\in \mathcal{F}_p$,
we include $T - 1$ auxiliary features $\mathcal{F}_p^i$ as follows:
$$
\mathcal{F}_{p}^i =\{f^{i,t}(h^T, m^T)~|~f^{i,t}(h^T, m^T) = \mathbf{1}_{\{(h^t,m^t)
= (\bar{h}^{i,t}, \bar{m}^{i,t})\}}, ~\mbox{for}~ t= 1,\ldots,T\},
$$
ensuring matching of full-length and prefixes for the path based features.
For the modified set of equality features
$\mathcal{F} = \mathcal{F}_{d} \cup (\bigcup\limits_{i\in\mathcal{F}_{p}}\mathcal{F}_{p}^i)$.
and an arbitrary set of inequality features $\mathcal{G}$
AREA converges in one iteration.
\end{theorem}

Please see \cref{sec:proof-thm:path-reward} for proof.
Note that by following a similar argument as in the proof of \Cref{thm:complexity},
one can show that features included in $\mathcal{F}_p^i$ will not undermine the
polynomial complexity. One just needs to keep track of the prefix overlap between
the current conditioning sequence and the support of each path-based feature, in order
to compute the associated $Z_{\bm\lambda}$.

\else
See \cite{preprint}
for details.
\fi

For more general cases, the convergence of AREA algorithm is subtle. Note that the human estimation problem (\ref{eqn:step1}) depends on the machine policy $Q(m^T\|h^T)$ used.
Thus given $Q(m^T \| h^T)$ at
the current iteration one can determine the associated
model for human behavior $h^*(Q, \mathbf{c}_f, \mathbf{c}_g)$
which may in turn change the optimal
machine policy. This makes
the analysis of convergence difficult.
In order to facilitate the convergence, we propose introducing
an additional inequality
constraint to the human estimation problem (\ref{eqn:step1}).

During the $n$th iteration, given the previously obtained $\hat{P}^{(n-1)}$ and $\hat{Q}^{(n)}$
we shall include the following {\em step-dependent} inequality constraint in $\mathcal{G}$.
Let
$
g^{0,(n)}(h^T, m^T) = -\log \hat{Q}^{(n)}(m^T\|h^T) + \gamma r(h^T, m^T),
$
and let $c_g^{0,(n)} = E_{\hat{P}^{(n-1)}\hat{Q}^{(n)}}[g^{0, (n)}(H^T, M^T)]$,
then on AREA iteration $n$ we require that
$
E_{P\hat{Q}^{(n)}}[g^{0, (n)}(H^T, M^T)] \geq c_g^{0,(n)}.
$

Let us define a sequence $\{L^{(n)}\}$ of entropy regularized expected rewards
across iterations, i.e.,
%$
%L^{(n)} := \mathbb{H}_{\hat{P}^{(n)}\hat{Q}^{(n)}}(M^T\|H^T)
%+ \gamma E_{\hat{P}^{(n)}\hat{Q}^{(n)}}[r(H^T, M^T)].
%$
$
L^{(n)} := E_{\hat{P}^{(n)}\hat{Q}^{(n)}}[g^{0, (n)}(H^T, M^T)].
$
Then we have the following result.
{
\begin{theorem}\label{thm:convergence}
Consider the AREA algorithm optimizing a human-machine interactive process
with a fixed sets of equality/inequality constraints ${\cal F}$ and ${\cal G}$.
Suppose ${\cal G}$ is modified to ${\cal G}^{(n)}$
by adding the additional step-dependent inequality constraint
$E_{P\hat{Q}^{(n)}}[g^{0,(n)}(H^T, M^T)] \ge c_g^{0,(n)}$.
Then the modified AREA iterations generate a bounded nondecreasing
sequence $\{L^{(n)}\}$,
which must converge.
\end{theorem}
}

\ifappendix
The proof is included in \cref{proof:thm-convergence}.
\fi

\emph{Remark:}
Note that when the conditions in Theorem \ref{thm:complexity} holds true,
then $\hat{Q}^{(n)}$ takes independent actions once the path deviates from the support
of all path-based feature functions.
Thus the introduced step-dependent feature function can be written as:
$
g^{0,(n)}(h^T, m^T) = -\sum_{t = 1}^T \log \hat{Q}_t^{(n)}(m_t)
+ \sum_{i \in \mathcal{F}_{p} \cup \mathcal{G}_{p}}
\left( \sum_{t = 1}^T \log \hat{Q}_t^{(n)}(\bar{m}^i_t) - \log \hat{Q}^{(n)}(\bar{m}^{i,T}\|\bar{h}^{i,T}) \right)\mathbf{1}_{\{h^T = \bar{h}^{i, T}, m^T = \bar{m}^{i,T}\}},
$
%\begin{eqnarray*}
%\lefteqn{g^{0,(n)}(h^T, m^T) = -\sum_{t = 1}^T \log \hat{Q}_t^{(n)}(m_t)}\\
%& &  + \sum_{i \in \mathcal{F}_{p} \cup \mathcal{G}_{p}}
%    \left( \sum_{t = 1}^T \log \hat{Q}_t^{(n)}(\bar{m}^i_t) - \log \hat{Q}^{(n)}(\bar{m}^{i,T}\|\bar{h}^{i,T}) \right)\mathbf{1}_{\{h^T = \bar{h}^{i, T}, m^T = \bar{m}^{i,T}\}},
%\end{eqnarray*}
which is still a weighted sum of path-based functions and decomposable functions.
This in turn means that the added constraint is such that iteration steps will
still have the polynomial complexity shown in \Cref{thm:complexity}.

$L^{(n)}$ can be regarded
as a measure of the performance of the associated machine policy
$\hat{Q}^{(n)}$. Indeed if $\hat{Q}^{(n)}$ were a fixed point
of AREA recursion, then the optimal objective function of (\ref{eqn:step2}) would
have converged to $L^{(n)}$.
Also note that by further assuming that the feature and reward functions are
decomposable, we can characterize the performance for
the converging sequence $\{L^{(n)}\}$--see \ifappendix\Cref{sec:decomposable}\else\cite{preprint}\fi.

%% file: evaluation.tex
\section{Evaluation}\label{sec:evaluation}
%The numerical results should cover something missed in the theoretical
%results, including:
%\begin{enumerate}
%\item Robustness against sampling error in both human estimation
%and AREA convergence
%\item Convergence of $Q$
%\item Comparing AREA with some benchmark, for e.g., Q-learning
%\end{enumerate}
%In the experiment, we'll adopt leaky competitive accumulator
%in order to mimic the human decision-making process \cite{UsM01}.
%Thus following experiments could be considered:
%\begin{enumerate}
%\item One shot estimation of human behavior with different sample size.
%Causal log loss can be used to measure error between the estimated human
%model $\hat{P}$ and true human behavior $P^*$.
%\item AREA iterations when sampling error is introduced. First, need to test
%the convergence rate of the iteration. We can test different number of samples
%for each iteration, also different gamma.
%
%Second, use some benchmark to capture the performance of AREA
%algorithm. One option is to compare the final converging point
%of AREA with gradient based optimization method when there is only
%decomposable features. Or we can compare the AREA with,
%for example, Q-learning with finite-horizon.
%\end{enumerate}
In this section, we conduct a preliminary numerical evaluation of AREA using
synthetic human-machine interaction data based on the
Leaky Competing Accumulator (LCA) model, see \cite{UsM01}.
This non-linear noisy model is known to capture common human decision-making processes
driven by external stimuli.

\subsection{Robustness against sampling noise}
\label{sec:evaluation-robust}
Throughout the paper we have assumed no sampling noise when
estimating the moments of features. In practice the available
data may be limited or costly and thus noisy estimates are inevitable.
The robustness of maximum entropy inference
against such noise is mathematically characterized in Theorem 6 of \cite{ZBD13}.
In this section, we will explore the robustness of the AREA algorithm
to noise when the number of samples per iteration are limited.
\ifappendix The detailed set-up of the LCA model
for human-machine interactions is included in \Cref{sec:evaluation-setup}
\else See \cite{preprint} for detailed set-up\fi.

We consider a setting where $T = 30$, $\mathcal{H} = \{1, \ldots, 6\}$,
$\mathcal{M} = \{ 1, \ldots, 6\}$ and $\gamma = 2$.
The reward function is $r(h^T, m^T) = \sum_{t = 1}^T r_t(h_t, m_t)$, where
$r_t(h_t, m_t) = \mathbf{1}_{\{t~\textrm{mod}~5 = 0\}}\mathbf{1}_{\{h_t = 1\}}
+\mathbf{1}_{\{t~\textrm{mod}~5 \ne 0\}}\mathbf{1}_{\{h_t \ne 1\}}$, i.e.,
we are looking to funnel the human behavior to choosing
1 only at $t = 5, 10,\cdots,30.$
The features include the reward function itself, together with the number of times
human follows the machine
$f^1(h^T, m^T) = \sum_{t = 1}^T \mathbf{1}_{\{h_t = m_t\}}$,
and a `weighted' number of times of following occurs emphasizing
later times, i.e., $t = 5, 10, \cdots,30$, 
$f^2(h^T, m^T) = \sum_{t = 1}^T f^2_t(h_t, m_t)$, where
$f^2_t(h_t, m_t) = (\mathbf{1}_{\{t~\textrm{mod}~5 = 0\}} + 0.25 \mathbf{1}_{\{t~\textrm{mod}~5 \ne 0\}})\mathbf{1}_{\{h_t = m_t\}}$.
The challenge here is for the machine to learn to drive
human (nonlinear model) away from 1 and back to 1 periodically.

The results in Fig.~\ref{fig:convergence-samplesize-decomposable}
exhibit the convergence of the regularized reward function $L$ vs
the number of AREA steps, when different numbers of samples are used
to estimate the moments in AREA's inference step.
%The empirical $L$ function means that in the human estimation step,
%the solution relies on the empirical distribution obtained from the data
%instead of true $P^*(h^T\|m^T)$.
%Note that the $y$-axis ranges only from 96 to 105 for visibility.
Clearly, AREA converges almost immediately although
it exhibits variations when small samples ($\leq 100$) are used.
\begin{figure*}[!tbp]
\centering
\subfigure[Convergence of AREA vs. sample size of data collection]{
\includegraphics[height = 2in]{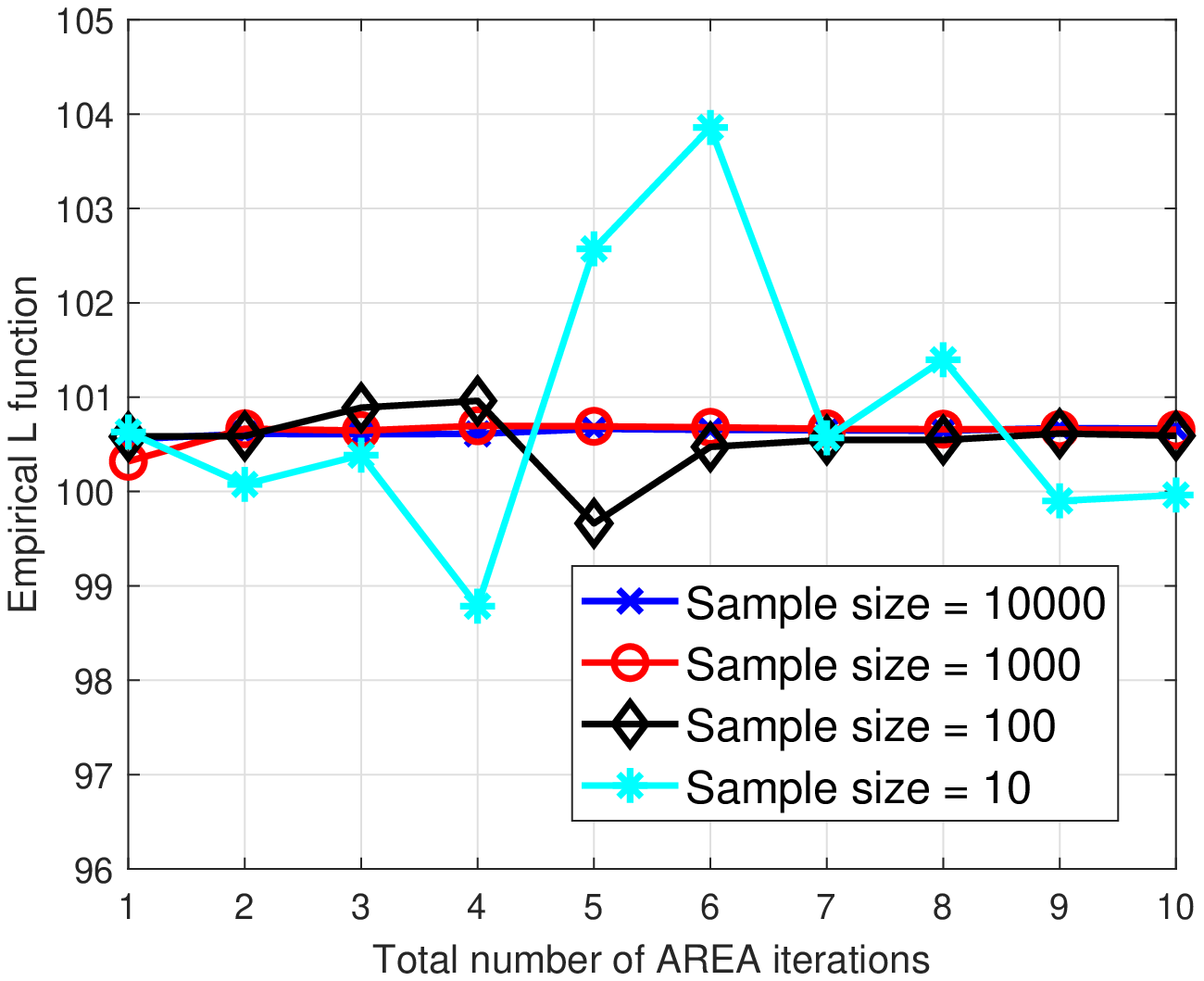}
\label{fig:convergence-samplesize-decomposable}
}
\subfigure[Average reward vs. number of samples observed]{
\includegraphics[height = 2in]{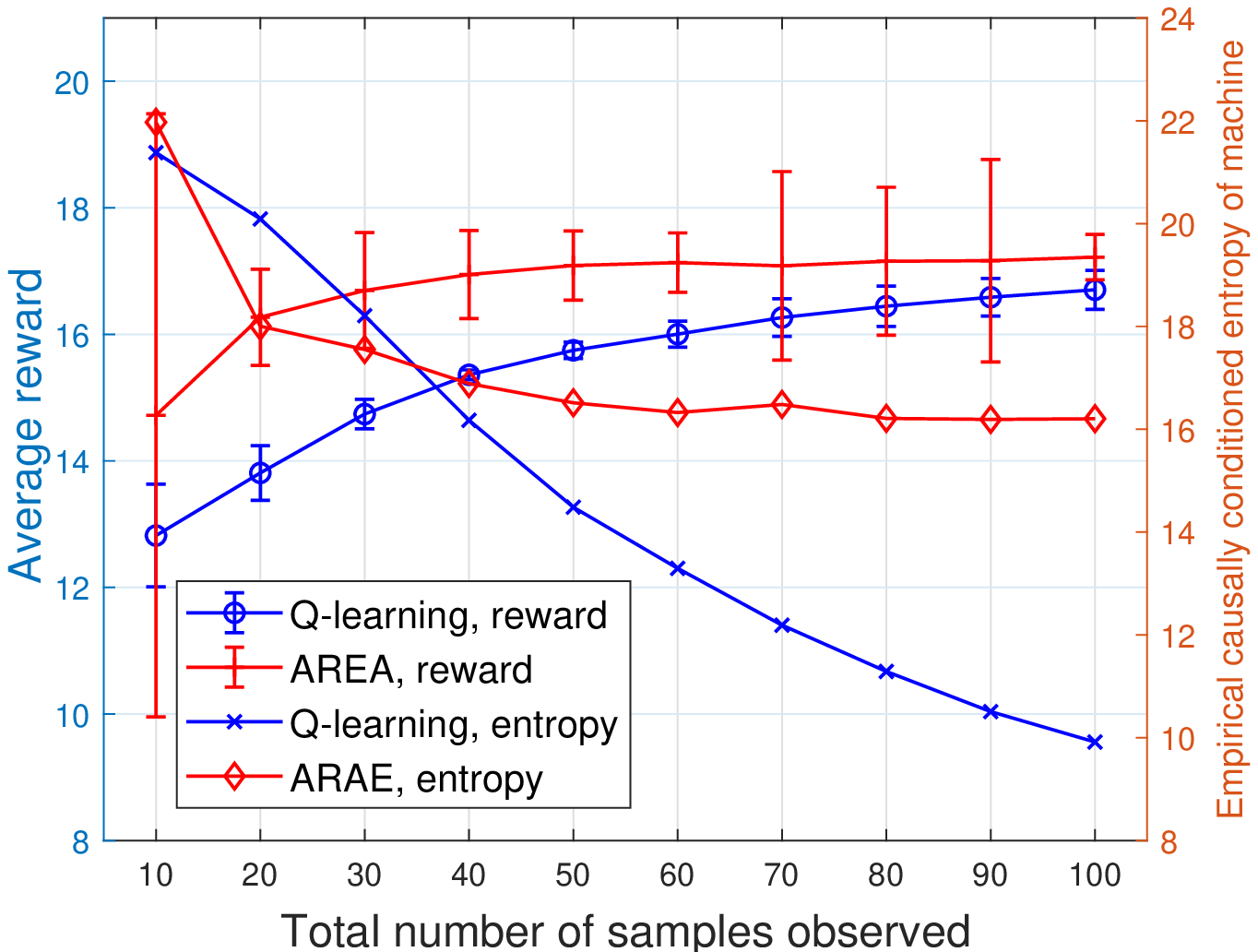}
\label{fig:reward-vs-Qlearning}
}
\vspace{-2em}
\end{figure*}
%\textbf{Another version of a path-based setting:}
%The simulation
%setting is simple, where we have $\mathcal{H} = \mathcal{M} = \{1,2,3\}$,
%$T = 3$
%and the reward function is defined as a path-based function
%with support $h^T = (1,1,1), m^T = (1,1,1)$. Coefficient $\gamma$ in
%(\ref{eqn:step2}) is set to be 2. The feature set $\mathcal{F}$ we use
%includes two decomposable ones $f^1(h^T, m^T), f^2(h^T, m^T)$, where
%$f^1_t(h_t, m_t) = 0.4\mathbf{1}_{\{t = 3\}}\mathbf{1}_{\{h_t = m_t\}}$,
%$f^2_t(h_t, m_t) = 0.1 t\mathbf{1}_{\{h_t = m_t\}}$,
%and one path-based one, with support $h^T = (2,2,2), m^T = (3,3,3)$.
%
%The results are shown in \Cref{fig:convergence-samplesize}.
%As one can see, when the sample size is greater than 100, the convergence
%is clean, and collecting more samples improves the performance of AREA
%in terms of the value of $L$ function because in the human estimation
%is more stable and accurate.
%On the other hand, if the number of samples
%is insufficient (for example, 10), chances are that the moments measured
%from the data are severely skewed. For example, none of the collected sample
%has nonzero reward, thus the human estimation step outputs an estimated model
%such that never follow the support of the reward function by mistake.
%Then by no means can the machine identify the `profitable' machine's action.
%\begin{figure}[!tbp]
%\centering
%\includegraphics[height = 3in]{./figures/convergence-samplesize}
%\caption{Convergence of AREA vs. sample size of data collection}\label{fig:convergence-samplesize}
%\end{figure}
\subsection{Performance in average reward and causally conditioned entropy}
Next we compare the performance of AREA to a simple Q-learning algorithm \cite{KLM96}
with finite memory. We shall compare the attained reward
and empirical causally conditioned entropy of the optimized machine
policies.
%The empirical causally conditional entropy
%is obtained by sampling $\hat{Q}(M_t|H^{t-1}, M^{t-1})$ from
%the interactions.
In this setting the humans' actions are viewed as the environment. Thus,
instead of `scoring' each action based on the most recent humans' response,
Q-learning scores each action based on the most recent $\tau$ interactions,
together with $t$ to accommodate the transient nature of the process, i.e., it
keeps track of
$Q( (h_{t- \tau + 1}, \ldots, h_t, m_{t- \tau + 1}, \ldots, m_t), t+1 , m_{t+1})$.

At time $t$, the machine chooses an action using a softmax of the $Q$ function
given the latest interaction history
$(h_{t- \tau + 1}, \ldots, h_t, m_{t- \tau + 1}, \ldots, m_t)$
and $t + 1$, and then updates the $Q$ function accordingly. We shrink the state
space to $|\mathcal{H}| = |\mathcal{M}|=3$ and $T = 20$ so the $Q$ function fits
in the memory and also change $\gamma$ to 4 to put more emphasis on reward.
We shall consider the same rewards and features
as in \Cref{sec:evaluation-robust}.
%Interesting machine learning problems in real life are usually of large scale but the amount of data available is very limited.  Thus it's important to evaluate the sample efficiency
of AREA.
We will let both algorithms complete 100 `interactions' with our synthetic human model.
For AREA, we collect 10 human-machine interaction samples per AREA iteration,
and run 10 iterations in total.
For Q-learning we also allow a total of 100 interactions. We set $\tau = 1$ since
further experiments show that greater $\tau$ impairs the performance
of Q-learning for it requires more samples to learn.
The detailed setup for Q-learning can be found in \ifappendix\Cref{sec:evaluation-setup}\else \cite{preprint}\fi.
We kept track of the average reward obtained,
estimated causally conditioned entropy of machine obtained for both algorithms after
integrating the first $n$ samples.
We run the simulation 5 rounds to obtain the average,
and the results,
together with the 90\% confidence intervals
are shown in \Cref{fig:reward-vs-Qlearning}.
These representative results suggest that typically
AREA algorithm is very efficient, delivering
higher rewards than Q-learning while
at the same time realizing (as desired) higher machine
policy entropy with a very limited number of samples.

\section{Conclusions}
The paper proposes a general data-driven framework to optimize possibly complex
human-machine interaction processes. At the core is the AREA algorithm
which jointly solves the problem of
estimating a model for human behaviour and optimizing the machine
policy based on a constrained maximum entropy estimation.
An underlying goal is to enable the integration of domain-specific knowledge
regarding relevant interaction characteristics or known
human biases by matching the observed moments of feature functions.
The paper details the formal optimization problems and solutions
underlying the AREA algorithm and explores a modification to significantly
reduce the complexity when the feature and reward functions
are path-based and/or decomposable.
The setting considered is fairly general, allowing one to incorporate
human-machine interactions with long memory.
The characterization of AREA is provided in terms of
($i$) its space and time complexity, and
($ii$) its convergence in various settings.
A simple numerical evaluation is used to demonstrate
the robustness of AREA to noise when sample sizes are limited,
along with a performance comparison to Q-learning.
The analysis and simple validation suggest that
AREA may achieve most of its gains in one iteration particularly
if sufficient domain specific features/rewards are properly integrated.

%% file: appendix.tex
\appendix

\section{Solution to dual of Problem (\ref{eqn:step1})}\label{sec:solution-estimation-dual}
Theorem 4 in \cite{ZBD13} shows the strong duality of Problem (\ref{eqn:step1}).
Therefore, the optimal dual induces the optimal primal solution.
If we find a $\bm{\lambda}^*$ minimizing (\ref{eq:estimate-dual}),
the estimated human model $\hat{P}$ is given by $P_{\bm\lambda^*}$.
The dual problem can be solved via a subgradient-based algorithm.
In particular, if we use an
adaptive learning rate $\eta^{(n)}\in \mathbb{R}^+$, the dual variable should be updated by
\begin{align}
&\bm{\lambda}_f^{(n+1)}\leftarrow \bm{\lambda}_f^{(n)} - \eta^{(n)}(E_{P_{\bm\lambda}Q}[\mathbf{f}(H^T,M^T)]-\mathbf{c}_f),\nonumber\\
&\bm{\lambda}_g^{(n+1)}\leftarrow \max\{0, \bm{\lambda}_g^{(n)} - \eta^{(n)} (E_{P_{\bm\lambda}Q}[\mathbf{g}(H^T,M^T)]-\mathbf{c}_g)\}, \label{eq:dual-update}
\end{align}
where $\mathbf{c}_f = E_{P^*Q}[\mathbf{f}(H^T,M^T)]$ are the moments of the feature functions
associated with the equality constraints obtained from the human-machine interaction data in
the inference step, and the gradients are computed using
the recursive form defined in \Cref{thm:solution}.
Then the sequence $\{\boldsymbol{\lambda}^{(n)}\}$ converges to the optimal dual $\boldsymbol{\lambda}^*$.

\section{Proof of Theorem \ref{thm:solution-machine}}
\label{sec:proof-thm:solution-machine}
The machine optimization problem (\ref{eqn:step2}) can be shown to be concave in $Q$ thus
one can directly solve it via first-order optimality conditions.
Considering the variables to be $\{Q(m_t|h^{t-1}, m^{t-1}),
 t= 1,2,\ldots T, h^{t-1} \in \mathcal{H}^{t-1}, m^t \in \mathcal{M}^t\}$,
the Lagrangian associated with Problem (\ref{eqn:step2}) can be written as:
\begin{align*}
\Lambda(Q, \beta) = & \mathbb{H}_{\hat{P}Q}(M^T\|H^T) + \gamma E_{\hat{P}Q}[r(H^T, M^T)] \\
& - \sum_{\substack{1\le t\le T \\ h^{t - 1} \in \mathcal{H}^{t - 1} \\ m^{t-1} \in \mathcal{M}^{t-1}}} \beta(h^{t-1}, m^{t-1})\left(1 - \sum_{m_t}Q(m_t | h^{t-1}, m^{t-1})\right),
\end{align*}
where $\beta(h^{t-1}, m^{t-1})$ for $t=1,\ldots,T$ denote dual variables associated with
the respective normalization constraints $\sum_{m_t}Q(m_t | h^{t-1}, m^{t-1}) = 1$.
By differentiating the Lagrangian we have
\begin{eqnarray*}
\lefteqn{\nabla_{Q(m_t|h^{t-1}, m^{t-1})}\Lambda(Q, \beta) = }\\
& & \beta(h^{t-1}, m^{t-1}) + \hat{P}Q(h^{t-1}, m^{t-1})\left(-\log Q(m_t|h^{t-1}, m^{t-1}) - 1 \right.\\
& & + \left.\mathbb{H}_{\hat{P}Q}(M^T\|H^T | h^{t-1}, m^t) + \gamma E_{\hat{P}Q}[r(H^T, M^T)|h^{t-1}, m^t] \right),
\end{eqnarray*}
where $\mathbb{H}_{\hat{P}Q}(M^T\|H^T | h^{t-1}, m^t)$ is the further conditioned, causally conditioned
entropy, defined as:
$$
\mathbb{H}_{\hat{P}Q}(M^T\|H^T | h^{t-1}, m^t) := E_{\hat{P}Q}[-\log Q(M^T\|H^T)~|~H^{t-1} = h^{t-1}, M^{t-1} = m^{t-1}].
$$
%Equating above gradient to 0 yields:
%\begin{eqnarray*}
%Q(m_t|h^{t-1}, m^{t-1}) \propto \exp \left\{ \mathbb{H}_{\hat{P}Q}(M^T\|H^T | h^{t-1}, m^t)
% + \gamma E_{\hat{P}Q}[r(H^T, M^T)|h^{t-1}, m^t] \right\}
%\end{eqnarray*}
After plugging $Y_\gamma$ defined recursively in \Cref{thm:solution-machine},
and setting $\beta(h^{t-1}, m^{t-1}) = \hat{P}Q(h^{t-1}, m^{t-1}) + \log Y_\gamma(h^{t-1}, m^{t-1})\hat{P}Q(h^{t-1}, m^{t-1})$,
we can show that $\nabla_{Q(m_t|h^{t-1}, m^{t-1})}\Lambda(Q, \beta) = 0.$ Thus the optimal solution is achieved.

\section{Proof of Theorem \ref{thm:complexity}}
\label{sec:proof-thm:complexity}
Before proving \Cref{thm:complexity}, let us first
consider a simpler case where only decomposable features are
included in Problem (\ref{eqn:step1}).
The following corollary to \Cref{thm:solution} characterizes a case where the complexity of the solution
is polynomial in $T$.

\begin{lemma}\label{cor:decomposable}
Suppose the machine's policy is given by a (possibly time-inhomogeneous) one-step Markov process, i.e.,
$Q(m_t |h^{t-1}, m^{t - 1}) = Q(m_t|h_{t-1}, m_{t-1}),~\forall t, m^{t-1}, h^{t-1}$, and all
feature functions are decomposable, i.e.,
$\mathbf{f}(h^T, m^T) = \sum_{t = 1}^T \mathbf{f}_{t}(h_t, m_t)$,
and
$\mathbf{g}(h^T, m^T) = \sum_{t = 1}^T \mathbf{g}_{t}(h_t, m_t)$.
Then the solution to the human estimation problem is given, by the following procedure
over a given dual $\boldsymbol{\lambda} = (\boldsymbol{\lambda}_f, \boldsymbol{\lambda}_g)$:
\begin{align*}
& && Z_{\boldsymbol{\lambda}}(h_t|m_t) = \left\{
\begin{array}{ll}
e^{(\boldsymbol{\lambda}_f)^T \mathbf{f}_t(h_t, m_t) + (\boldsymbol{\lambda}_g)^T \mathbf{g}_t(h_t, m_t)+ \sum_{m_{t+1}}Q({m}_{t+1}|h_{t}, m_{t})\log Z_{\boldsymbol{\lambda}}(m_{t+1})} & t < T \\
e^{(\boldsymbol{\lambda}_f)^T \mathbf{f}_T(h_T, m_T) + (\boldsymbol{\lambda}_g)^T \mathbf{g}_T(h_T, m_T)} & t = T
\end{array}
\right.,
\end{align*}
where
$$
Z_{\boldsymbol{\lambda}}(m_t) = \sum_{h_{t}} Z_{\boldsymbol{\lambda}}(h_t|m_t),~\textrm{and}~P_{\bm{\lambda}}(h_t | m_t) = \frac{Z_{\boldsymbol{\lambda}}(h_t|m_t)}{Z_{\boldsymbol{\lambda}}(m_t)}.
$$
Moreover,
both the space and time complexity of establishing the distributional
model is $O(T|\mathcal{H}||\mathcal{M}|)$.
The complexity of carrying out each dual update is
$O(T(|\mathcal{F}|+ |\mathcal{G}|)|\mathcal{H}|^2|\mathcal{M}|^2)$
\end{lemma}

\begin{proof}
We'll prove that under such assumption, $Z_{\boldsymbol{\lambda}}(h_t|h^{t-1}, m^{t})$ in \cref{thm:solution} is given by:
\begin{align*}
Z_{\boldsymbol{\lambda}}(h_t|h^{t-1}, m^{t}) = & e^{(\boldsymbol{\lambda}_f)^T\sum_{\tau = 1}^{t - 1}\mathbf{f}_\tau(h_\tau, m_\tau) +
(\boldsymbol{\lambda}_g)^T\sum_{\tau = 1}^{t - 1}\mathbf{g}_\tau(h_\tau, m_\tau)} \\
&\times e^{(\boldsymbol{\lambda}_f)^T \mathbf{f}_t(h_t, m_t) + (\boldsymbol{\lambda}_g)^T \mathbf{g}_t(h_t, m_t) + \sum_{m_{t+1}}Q({m}_{t+1}|h_{t}, m_{t})\log Z_{\boldsymbol{\lambda}}(m_{t+1})}
\end{align*}
where $Z_{\boldsymbol{\lambda}}(m_{t+1})$ is given as in Lemma \ref{cor:decomposable}.

The above equation implies that:
\begin{eqnarray*}
{P}_{\boldsymbol{\lambda}}(h_t|h^{t-1}, m^{t}) & = & \frac{Z_{\boldsymbol{\lambda}}(h_t|h^{t-1}, m^{t})}{Z_{\boldsymbol{\lambda}}(h^{t-1}, m^{t})}
  =  \frac{Z_{\boldsymbol{\lambda}}(h_t|m_t)}{Z_{\boldsymbol{\lambda}}(m_t)}
  =  P_{\boldsymbol{\lambda}}(h_t | m_t),
\end{eqnarray*}
and the Markov property follows.

For $t = T$, the identity holds true trivially. Now suppose it is true for $t + 1$. Then according to \Cref{thm:solution}, for $t< T$
\begin{eqnarray*}
& &\lefteqn{Z_{\boldsymbol{\lambda}}(h_t|h^{t-1}, m^{t})} \\
& = & e^{\sum_{m_{t+1}}Q({m}_{t+1}|h_{t}, m_{t})\log Z_{\boldsymbol{\lambda}}(h^{t}, m^{t+1})}\\
& = & e^{\sum_{m_{t+1}}Q({m}_{t+1}|h_{t}, m_{t})\log \sum_{h_{t + 1}} Z_{\boldsymbol{\lambda}}(h_{t + 1} | m_{t + 1}) e^{(\boldsymbol{\lambda}_f)^T\sum_{\tau = 1}^t\mathbf{f}_\tau(h_\tau, m_\tau) + (\boldsymbol{\lambda}_g)^T\sum_{\tau = 1}^t\mathbf{g}_\tau(h_\tau, m_\tau) }} \\
& = & e^{\sum_{m_{t+1}}Q({m}_{t+1}|h_{t}, m_{t})\log e^{(\boldsymbol{\lambda}_f)^T\sum_{\tau = 1}^t\mathbf{f}_\tau(h_\tau, m_\tau) + (\boldsymbol{\lambda}_g)^T\sum_{\tau = 1}^t\mathbf{g}_\tau(h_\tau, m_\tau)}} \\
& & \times e^{\sum_{m_{t+1}}Q(m_{t+1}|h_{t}, m_{t})\log \sum_{h_{t + 1}} Z_{\boldsymbol{\lambda}}(h_{t + 1} | m_{t + 1})}\\
& = & e^{(\boldsymbol{\lambda}_f)^T\sum_{\tau = 1}^{t-1}\mathbf{f}_\tau(h_\tau, m_\tau) + (\boldsymbol{\lambda}_g)^T\sum_{\tau = 1}^{t-1}\mathbf{g}_\tau(h_\tau, m_\tau)}\\
& & \times  e^{(\boldsymbol{\lambda}_f)^T \mathbf{f}_t(h_t, m_t) + (\boldsymbol{\lambda}_g)^T \mathbf{g}_t(h_t, m_t) + \sum_{m_{t+1}}Q(m_{t+1}|h_{t}, m_{t})\log \sum_{h_{t + 1}} Z_{\boldsymbol{\lambda}}(h_{t + 1} | m_{t + 1})}.
\end{eqnarray*}

Then when we compute the ratio $\frac{Z_{\boldsymbol{\lambda}}(h_t|h^{t-1}, m^{t})}{Z_{\boldsymbol{\lambda}}(h^{t-1}, m^{t})}$ the term
 $e^{(\boldsymbol{\lambda}_f)^T\sum_{\tau = 1}^{t-1}\mathbf{f}_\tau(h_\tau, m_\tau) + (\boldsymbol{\lambda}_g)^T\sum_{\tau = 1}^{t-1}\mathbf{g}_\tau(h_\tau, m_\tau)}$
cancels out.

For the complexity, it's easy to see that in total we need to compute
$T|\mathcal{H}||\mathcal{M}|$ probabilities. Thus the space complexity is $O(T|\mathcal{H}||\mathcal{M}|)$.
If the vector multiplication is viewed as a basic operation, then
computing each $Z_{\bm{\lambda}}(h_t|m_t)$ involves the sum of at most three
vector inner products, and evaluating of its exponentiation. Therefore,
the time complexity involved in establishing the distributional
model is also $O(T|\mathcal{H}||\mathcal{M}|)$.

When computing the expectation of the feature functions, note that
since all feature functions are decomposable,
for all $i$:
$$
E_{P_{\bm{\lambda}}Q}[f^i(H^T, M^T)] = \sum_{t = 1}^{T}E_{P_{\bm{\lambda}}Q}[f^i_{t}(H_t, M_t)].
$$
And
$$
E_{P_{\bm{\lambda}}Q}[f^i_{t}(H_t, M_t)] = \sum_{m_t \in \mathcal{M}}P_{\bm{\lambda}}Q(m_t)
\sum_{h_t\in\mathcal{H}}P_{\bm{\lambda}}(h_t | m_t)f_t(h_t, m_t).
$$
Suppose we already obtained $P_{\bm{\lambda}}Q(m_{t-1})$, then
$$
P_{\bm{\lambda}}Q(m_t) = \sum_{m_{t-1}\in\mathcal{M}} P_{\bm{\lambda}}Q(m_{t-1})
\sum_{h_{t-1}\in\mathcal{H}} P_{\bm{\lambda}}(h_{t - 1} | m_{t - 1}) Q(m_t | h_{t-1}, m_{t-1}).
$$
Note that the marginal distribution of $m_1$ is given by $P_{\bm{\lambda}}Q(m_1)$ = $Q(m_1)$, 
which is already available.
Thus we can compute $E_{P_{\bm{\lambda}}Q}[f^i_{t}(H_t, M_t)]$ from $t = 1$
to $t = T$ and store $P_{\bm{\lambda}}Q(m_t), \forall 1 < t \le T, m_t\in\mathcal{M}$.
Then it is straightforward that computing $E_{P_{\bm{\lambda}}Q}[f^i_{t}(H_t, M_t)]$
involves $|\mathcal{H}|^2|\mathcal{M}|^2$ operations, and computing $E_{P_{\bm{\lambda}}Q}[f^i(H^T, M^T)] $
is of complexity $O(T|\mathcal{H}|^2|\mathcal{M}|^2)$. Each dual update involves
evaluation of $E_{P_{\bm{\lambda}}Q}[f^i(H^T, M^T)],~\forall~i \in \mathcal{F}$,
and $E_{P_{\bm{\lambda}}Q}[g^i(H^T, M^T)],~\forall~i \in \mathcal{G}$,
thus has the complexity of $O(T(|\mathcal{F}| + |\mathcal{G}|)|\mathcal{H}|^2|\mathcal{M}|^2)$.
\end{proof}

Now let us assume that the equality and inequality constraint sets can
be each partitioned into two subsets:
$\mathcal{F} = \mathcal{F}_{p}\cup \mathcal{F}_{d}$,
and $\mathcal{G} = \mathcal{G}_{p}\cup \mathcal{G}_{d}$,
where $\mathcal{F}_{d}$ and $\mathcal{G}_{d}$ correspond to the decomposable features
and $\mathcal{F}_{p}$ and $\mathcal{G}_{p}$ correspond to the path-based features.
Moreover, the path-based features are:
$$
f^i(h^T, m^T) =
c_i\mathbf{1}_{\{ (h^T, m^T) = (\bar{h}^{i,T},\bar{m}^{i,T}) \}},
~i \in \mathcal{F}_{p}, ~\textrm{and}~~ g^i(h^T, m^T) =
c_i\mathbf{1}_{\{ (h^T, m^T) = (\bar{h}^{i,T},\bar{m}^{i,T}) \}},
~i \in \mathcal{G}_{p},
$$
while decomposable features are:
$$
f^i(h^T, m^T) = \sum_{t = 1}^T f_t^i(h_t, m_t),~i \in \mathcal{F}_{d},~\textrm{and}~~g^i(h^T, m^T) = \sum_{t = 1}^T g_t^i(h_t, m_t),~i \in \mathcal{G}_{d}.
$$
Also, the reward function is given by its path-based part
$
r^{i,p}(h^T, m^T) = c_i\mathbf{1}_{\{
(h^T, m^T) = (\bar{h}^{i, T},\bar{m}^{i, T})\}},
~~i\in \mathcal{R}_p,
$, together with a decomposable part $r^d(h^T, m^T) = \sum_{t = 1}^T r^d_t(h_t, m_t)$,
giving
$$
r(h^T, m^T) = \sum_{i\in \mathcal{R}_p}r^{i,p}(h^T, m^T) + \sum_{t = 1}^T r^d_t(h_t, m_t).
$$

First let us consider the human estimation problem.
Note that if the conditioning sequence is not a prefix of any path-based feature function (including
functions in $\mathcal{R}_p$),
the backward recursion in \cref{thm:solution} is equivalent to the case where we only have decomposable feature functions.

Without loss of generality, consider decomposable feature functions given by:
$$
\mathbf{f}^{d}(h^T, m^T) = \sum_{t = 1}^T \mathbf{f}^{d}_t(h_\tau, m_\tau),
$$
and
$$
\mathbf{g}^{d}(h^T, m^T) = \sum_{t = 1}^T \mathbf{g}^{d}_t(h_\tau, m_\tau).
$$
{\color{black}
We shall let $\boldsymbol{\lambda}_{f}^d$ be the dual variable corresponding to the decomposable equality constraints,
$\boldsymbol{\lambda}_g^d$ that corresponding to the decomposable inequality constraints,
and $\lambda_r$ that corresponding to the reward function.
Note that when we establish the distributional
model, functions in $\mathcal{R}_p$ together with
$r^d(h^T, m^T)$ can be regarded as individual `feature'
functions, which share the same dual variable $\lambda_r$.
It follows from Lemma \ref{cor:decomposable} that
if $(h^{t}, m^{t}) \ne (\bar{h}^{i,t}, \bar{m}^{i,t}),~\forall i\in\mathcal{F}_{p} \cup \mathcal{G}_{p} \cup \mathcal{R}_p$,
we have
$$Z_{\boldsymbol{\lambda}}(h_t | h^{t - 1}, m^{t}) = Z_{\boldsymbol{\lambda}}(h_t|m_t)e^{(\boldsymbol{\lambda}^d_f)^T\sum_{\tau = 1}^{t - 1} \mathbf{f}_\tau^{d}(h_\tau, m_\tau) + (\boldsymbol{\lambda}^d_g)^T\sum_{\tau = 1}^{t - 1} \mathbf{g}_\tau^{d}(h_\tau, m_\tau)
+ \lambda_r\sum_{\tau = 1}^{t - 1}r^d_\tau(h_\tau, m_\tau)},$$
where $Z_{\boldsymbol{\lambda}}(h_t|m_t)$ is given by the recursion specified in Lemma \ref{cor:decomposable},
with $r^d(h^T, m^T)$ as a feature function.
Let us denote the set of
machine actions at time $t$ that stay on at least one path-based feature function's support,
by $\mathcal{M}_t^p(h^{t-1},m^{t - 1}) = \{m_t | \exists i\in\mathcal{F}_{p}\cup \mathcal{G}_{p} \cup \mathcal{R}_p~s.t.~h^{t-1} = \bar{h}^{i,t-1}, m^t = \bar{m}^{i,t}\}$
and a similar set of human actions, by
$\mathcal{H}_t^p(h^{t-1}, m^t) = \{h_t | \exists i\in\mathcal{F}_{p}\cup \mathcal{G}_{p}\cup \mathcal{R}_p~s.t.~h^t = \bar{h}^{i,t}, m^t = \bar{m}^{i,t}\}$. For $\bar{h}^{i,t}, \bar{m}^{i,t}$,
the backward recursion in \cref{thm:solution} becomes following:
\begin{eqnarray}
&& Z_{\boldsymbol{\lambda}}(\bar{h}_{i,t} | \bar{h}^{i,t - 1}, \bar{m}^{i,t}) \nonumber\\
& = & e^{\sum_{m_{t + 1}} Q(m_{t + 1}| \bar{h}^{i,t}, \bar{m}^{i,t}) \log\sum_{h_{t+1}} Z_{\boldsymbol{\lambda}}(h_{t + 1} | \bar{h}^{i,t}, (\bar{m}^{i,t}, m_{t+1}))} \nonumber\\
& = & \exp(\underbrace{\sum_{m_{t+1} \in \mathcal{M}_{t+1}^p(\bar{h}^{i,t}, \bar{m}^{i,t})}Q(m_{t+1}| \bar{h}^{i,t}, \bar{m}^{i,t})\log \sum_{h_{t+1}} Z_{\boldsymbol{\lambda}}(h_{t + 1} | \bar{h}^{i,t}, (\bar{m}^{i,t}, m_{t+1}))}_{:= A} \nonumber\\
&& + \underbrace{\sum_{m_{t+1} \notin \mathcal{M}_{t+1}^p(\bar{h}^{i,t}, \bar{m}^{i,t})}Q(m_{t+1}| \bar{h}^{i,t}, \bar{m}^{i,t})\log\sum_{h_{t+1}} Z_{\boldsymbol{\lambda}}(h_{t + 1} | \bar{h}^{i,t}, (\bar{m}^{i,t}, m_{t+1}))}_{:= B}) \nonumber\\
& = & \exp(A + B).
\end{eqnarray}
From the result of Lemma \ref{cor:decomposable},
\begin{flalign*}
\lefteqn{A = \sum_{\substack{ m_{t+1} \\ \in \mathcal{M}_{t+1}^p(\bar{h}^{i,t}, \bar{m}^{i,t})}}Q(m_{t+1}| \bar{h}^{i,t}, \bar{m}^{i,t})
\log\Bigg(\sum_{\substack{ h_{t + 1} \\ \in \mathcal{H}_{t+1}^p(\bar{h}^{i,t}, (\bar{m}^{i,t}, m_{t+1}))}} Z_{\boldsymbol{\lambda}}(h_{t + 1} | \bar{h}^{i,t}, (\bar{m}^{i,t}, m_{t+1}))} \nonumber\\
& && +  \left.\sum_{h_{t + 1}\notin \mathcal{H}_t^p(\bar{h}^{i,t}, (\bar{m}^{i,t}, m_{t+1}))} Z_{\boldsymbol{\lambda}}(h_{t + 1} | m_{t+1})\exp\left((\boldsymbol{\lambda}^d_f)^T\sum_{\tau = 1}^t \mathbf{f}^{d}_\tau(\bar{h}_{i,\tau}, \bar{m}_{i,\tau}) \right.\right. \\
& && \left. + (\boldsymbol{\lambda}^d_g)^T\sum_{\tau = 1}^t \mathbf{g}^{d}_\tau(\bar{h}_{i,\tau}, \bar{m}_{i,\tau}) + \lambda_r\sum_{\tau = 1}^t r^d_\tau(\bar{h}^i_\tau, \bar{m}^i_\tau)\right)\Bigg).
\end{flalign*}

\begin{eqnarray*}
\lefteqn{B =  \sum_{m_{t+1} \notin \mathcal{M}_{t+1}^p(\bar{h}^{i,t}, \bar{m}^{i,t})}Q(m_{t+1}| \bar{h}^{i,t}, \bar{m}^{i,t})}\\
& & \times\log\sum_{h_{t+1}}
Z_{\boldsymbol{\lambda}}(h_{t + 1}|m_{t + 1})e^{(\boldsymbol{\lambda}^d_f)^T\sum_{\tau = 1}^t \mathbf{f}^{d}_\tau(\bar{h}_{i,\tau}, \bar{m}_{i,\tau}) + (\boldsymbol{\lambda}^d_g)^T\sum_{\tau = 1}^t \mathbf{g}^{d}_\tau(\bar{h}_{i,\tau}, \bar{m}_{i,\tau}) + \lambda_r \sum_{\tau = 1}^t r^d_\tau(\bar{h}_{i,\tau}, \bar{m}_{i,\tau})}.
\end{eqnarray*}
Note that $B$ solely depends on the result of the case where there are only decomposable features.
The additional
complexity introduced is in the computation of $A$, 
which is determined by the number of nonzero path-based features
after current step. The key insight is that we only need to track $A$ for a prefix where there is at least one nonzero
path-based feature function, and the set of possible choices of such prefixes forms a tree where the number of leaf nodes is
at most $|\mathcal{F}_p| + |\mathcal{G}_p| + |\mathcal{R}_p|$. Then at each $t$,
we need to compute $A$ for at most $|\mathcal{F}_p| + |\mathcal{G}_p| + |\mathcal{R}_p|$ conditioning prefixes. Thus the
complexity of obtaining the whole distributional model is $O((|\mathcal{F}_p| + |\mathcal{G}_p| + |\mathcal{R}_p|) T |\mathcal{H}||\mathcal{M}|)$.

When computing the mean of a feature function $E_{P_{\boldsymbol{\lambda}}Q}[f^i(H^T, M^T)]$, we have two different cases:
\begin{enumerate}
\item If $f^i(h^T, m^T)$ is a path-based feature with support $\bar{h}^{i,T}, \bar{m}^{i,T}$. Then $E_{P_{\boldsymbol{\lambda}}Q}[f^i(H^T, M^T)] = P_{\boldsymbol{\lambda}}Q(\bar{h}^T, \bar{m}^T) = c_i\prod_{t  = 1}^TQ(\bar{m}^i_{t} | \bar{h}^{i, t - 1}, \bar{m}^{i, t - 1})P(\bar{h}^i_t | \bar{h}^{i, t - 1}, \bar{m}^{i,t})$. This requires at most $T$ multiplications.

\item If $f^i(h^T, m^T)$ is a decomposable feature then the associated moment can be written as
$$
E_{P_{\boldsymbol{\lambda}}Q}[f(H^T, M^T)] = \sum_{t = 1}^TE_{P_{\boldsymbol{\lambda}}Q}[f^i_t(H_t, M_t)].
$$
Let us define a stopping time $T_D$ w.r.t. $(H^T, M^T)$ such that $T_D := \min~\{t ~|~1\le t \le T, (H^t, M^t) \ne (\bar{h}^{i,t} , \bar{m}^{i,t}),~\forall i\in \mathcal{F}_{p} \cup \mathcal{G}_{p} \cup \mathcal{R}_p\}$. That is, $T_D$ is the first time when the realization of interaction deviates from supports of all path-based feature functions,
including the path-based component of the reward. 
Then based on the value of $T_D$,
we can partition $E_{P_{\boldsymbol{\lambda}}Q}[f^i_t(H_t, M_t)]$ as follows:
\begin{eqnarray}
E_{P_{\boldsymbol{\lambda}}Q}[f_t(H_t, M_t)] & = & P_{\boldsymbol{\lambda}}Q(T_D \le t)E_{P_{\boldsymbol{\lambda}}Q}[f^i_t(H_t, M_t)|T_D \le t] \nonumber\\
& & + P_{\boldsymbol{\lambda}}Q(T_D > t)E_{P_{\boldsymbol{\lambda}}Q}[f^i_t(H_t, M_t)|T_D > t].
\end{eqnarray}
After deviating from all supports, i.e., when $T_D \le t$, the distribution is the same as the case
where only decomposable features functions have been included.
Thus $E_{P_{\boldsymbol{\lambda}}Q}[f^i_t(H_t, M_t)|T_D \le t]$ can be easily
obtained within $O(T|\mathcal{H}||\mathcal{M}|)$ computations,
by taking advantage of the one-step Markov property.
Also, the distribution of the stopping time $T_D$ can be computed as:
\begin{eqnarray*}
P_{\boldsymbol{\lambda}}Q(T_D \le t) && = 1 - P_{\boldsymbol{\lambda}}Q(T_D > t) \\
&& = 1 - \sum_{i\in \mathcal{F}_{p} \cup \mathcal{G}_{p} \cup \mathcal{R}_p} P_{\boldsymbol{\lambda}}Q(H^t = \bar{h}^{i,t}, M^t = \bar{m}^{i,t}) \\
&& = 1 - \sum_{i\in \mathcal{F}_{p} \cup \mathcal{G}_{p} \cup \mathcal{R}_p}\prod_{\tau = 1}^t P_{\boldsymbol{\lambda}}(\bar{h}_{i,\tau} | \bar{h}^{i, {\tau-1}}, \bar{m}^{i, \tau})Q( \bar{m}_{i,\tau}|\bar{h}^{i, {\tau-1}}, \bar{m}^{i, \tau - 1}).
\end{eqnarray*}
At most it requires $T(|\mathcal{F}_p| + |\mathcal{G}_p| + |\mathcal{R}_p|)$ computations. Same computation complexity is expected when computing $P_{\boldsymbol{\lambda}}Q(T_D > t)$.
For $E_{P_{\boldsymbol{\lambda}}Q}[f^i_t(H_t, M_t) | T_D > t]$, we have:
\begin{eqnarray*}
E_{P_{\boldsymbol{\lambda}}Q}[f^i_t(H_t, M_t) | T_D > t] && = \sum_{i \in \mathcal{F}_{p} \cup \mathcal{G}_{p} \cup \mathcal{R}_p }f^i_t(\bar{h}_{i,t},\bar{m}_{i,t})P_{\boldsymbol{\lambda}}Q(H_t = \bar{h}_{i,t}, M_t = \bar{m}_{i,t}| T_D > t) \\
&& = \sum_{i \in \mathcal{F}_{p} \cup \mathcal{G}_{p} \cup \mathcal{R}_p}f^i_t(\bar{h}_{i,t},\bar{m}_{i,t})\frac{P_{\boldsymbol{\lambda}}Q(H^t = \bar{h}^t_i, M^t = \bar{m}^t_i)}{P_{\boldsymbol{\lambda}}Q(T_D > t)}.
\end{eqnarray*}
It's easy to observe that it requires $O(T(|\mathcal{F}_p| + |\mathcal{G}_p| + |\mathcal{R}_p|))$ computations, too. Then the computation complexity to compute the sum is $O(T^2(|\mathcal{F}_p| + |\mathcal{G}_p| + |\mathcal{R}_p|))$.
\end{enumerate}

Exactly the same complexity is obtained when
computing functions in $\mathcal{G}$, $\mathcal{R}_p$ and $r^d(h^T, m^T)$.
Then the time complexity of one dual update will be given by the maximum of the two cases, as well as the the time to establish the
distributional model, thus is given by $O(T(|\mathcal{F}_p| + |\mathcal{G}_p| + |\mathcal{R}_p|) \max(T, |\mathcal{H}||\mathcal{M}|))$.

For the machine optimization problem, we do not need to carry out the
dual update, since $\gamma$ is fixed throughout the iterations. Therefore we
only need to establish the distributional model $\hat{Q}$. By viewing
the path-based part of the reward function as the path-based `feature'
in the machine optimization problem, we can easily conclude that both the space and time
complexity in obtaining the machine's policy $\hat{Q}$ is
$O((|\mathcal{F}_p| + |\mathcal{G}_p| + |\mathcal{R}_p|)T|\mathcal{H}||\mathcal{M}|)$. Moreover, as long as
the initial machine policy is such that after
deviating from the union of the supports of all path-based feature functions, it is one-step Markov,
i.e. $\hat{Q}^{(0)}(m_t | m^{t - 1}, h^{t - 1}) = \hat{Q}^{(0)}(m_t | m_{t - 1}, h_{t - 1})$
when $(m^{t - 1}, h^{t - 1}) \ne (\bar{m}_i^{t - 1}, \bar{h}_i^{t - 1}),~\forall i\in\mathcal{F}_{p} \cup \mathcal{G}_{p} \cup \mathcal{R}_p$,
all the assumptions introduced in \Cref{thm:complexity} are
satisfied throughout the AREA iterations.
A uniform random $\hat{Q}^{(0)}$ is a special case satisfying that condition.
}

Therefore the complexity of AREA algorithm is polynomial in $T$ as long
as the number of dual updates is limited in the human estimation problem.

\section{Convergence of AREA under sufficient statistics}\label{sec:sufficient-stats}
An implication of \cref{thm:solution} is that, under our maximum entropy framework,
estimated human models will be of the form given in the theorem for a given value of
$\boldsymbol{\lambda}$. We refer to such distributions as causally conditioned Gibbs distributions
formally defined as follows:

\begin{definition}
Given the set of constraints $\mathcal{F}$, $\mathcal{G}$ and underlying machine policy $Q(m^T\|h^T)$,
we define the associated \emph{causally conditioned Gibbs distributions} as
\begin{align}
\mathcal{P}_g(Q, \mathcal{F}, \mathcal{G}) := & \{~P(h^T\|m^T)~|~
\exists \boldsymbol{\lambda} \in \mathbb{R}^{|{\cal F}|} \times  \mathbb{R}_{-}^{|{\cal G}|}~\nonumber\\
& ~~~~~~~~~~~~~~~~~~~s.t.~  P(h_t | h^{t-1}, m^t) = P_{\boldsymbol{\lambda}}(h_t | h^{t-1}, m^t) ~\mbox{for}~ t=1,\ldots.T
~\},
\end{align}
where $P_{\boldsymbol{\lambda}}(h_t | h^{t-1}, m^t)$ is as given in \Cref{thm:solution}.
That is, each element in $\mathcal{P}_g(Q, \mathcal{F}, \mathcal{G})$ is a causally conditioned distribution
of the form given in \Cref{thm:solution} for a $\boldsymbol{\lambda} = (\boldsymbol{\lambda}_f, \boldsymbol{\lambda}_g)$.
\end{definition}

\emph{Remark:} According to \emph{Hammersley-Clifford Theorem} in \cite{HaC71},
if the human behavior $P^*(h^T\|m^T)$ has full support, i.e., there is no
$(h^T, m^T)\in \mathcal{H}^T \times \mathcal{M}^T$ such that $P^*(h^T\|m^T) = 0$, and machine's policy $Q(m^T\|h^T)$
also has full support, then there exists a pair of finite sets of constraint $\mathcal{F}^*$
and $\mathcal{G}^*$ such that
the true human behavior is in the associated causally conditioned Gibbs distribution set,
i.e., $P^*(h^T\|m^T) \in \mathcal{P}_g(Q, \mathcal{F}^*, \mathcal{G}^*)$.

Then if the features in human estimation problem are rich enough,
the following theorem captures the convergence of AREA.
\begin{theorem}\label{thm:sufficient-stats}
If the feature sets $\mathcal{F}$ and $\mathcal{G}$ and initial machine policy $\hat{Q}^{(0)}$ are such that,
$$
P^*(h^T\|m^T) \in
\mathcal{P}_g(\hat{Q}^{(0)},\mathcal{F}, \mathcal{G}) \cap \mathcal{P}_g(m^*(m(P^*)),\mathcal{F}, \mathcal{G})
$$
AREA algorithm converges after the first iteration.
\end{theorem}

\begin{proof}
When $P^*(h^T\|m^T) \in \mathcal{P}_g(\hat{Q}^{(0)},\mathcal{F}, \mathcal{G})$,
then $P^*(h^T\|m^T)$ can be parameterized by some $\bm\lambda^* := (\bm\lambda_f^*, \bm\lambda_g^*)$
and will be the solution to the human estimation problem,
based on the data produced under machine policy $\hat{Q}^{(0)}$.
Now, given $\hat{P}^{(0)} = P^*$,
the machine optimization problem generates $\hat{Q}^{(1)} = m^*(m(P^*))$.
However since we have that
$P^*(h^T\|m^T) \in \mathcal{P}_g(m^*(m(P^*)),\mathcal{F}, \mathcal{G})$, again we can ensure that
$\hat{P}^{(1)} = P^*$. By induction it is easy to see that $\hat{P}^{(n)} = P^*,~\forall n\ge 0$
and $\hat{Q}^{(n)} = m^*(m(P^*)),~\forall n \ge 1$. Thus AREA iterations will converge after the first
iteration.
\end{proof}

\section{Proof of Theorem \ref{thm:path-reward}}
\label{sec:proof-thm:path-reward}
At the $n$th iteration, when matching the moment of path-based features, we have:
$$
\hat{P}^{(n)}\hat{Q}^{(n-1)}(\bar{h}^{i,T}, \bar{m}^{i,T}) = 
P^{*}\hat{Q}^{(n-1)}(\bar{h}^{i,T}, \bar{m}^{i,T}),~\forall~i\in\mathcal{F}_{p}.
$$
After cancelling out $Q^{(n-1)}(\bar{m}^{i,T}\|\bar{h}^{i,T})$ on both sides we have
$$
\hat{P}^{(n)}(\bar{h}^{i,T}\|\bar{m}^{i,T}) = P^{*}(\bar{h}^{i,T}\|\bar{m}^{i,T}).
$$
If the feature set in Problem (\ref{eqn:step1}) also includes $\mathcal{F}_{p}^i$,
for all $i\in\mathcal{F}_{p}$,
by a similar argument we have that for all $i\in \mathcal{F}_{p}$ and $t = 1,2,\ldots, T$,
$$
\prod_{\tau = 1}^t P^{(n)}(\bar{h}_{i,t} | \bar{h}^{i, t-1}, \bar{m}^{i,t}) = \prod_{\tau = 1}^t P^{*}(\bar{h}_{i,t} | \bar{h}^{i, t-1}, \bar{m}^{i,t}).
$$
Thus $\hat{P}^{(n)}(\bar{h}_{i,t} | \bar{h}^{i, t-1}, \bar{m}^{i,t}) = P^{*}(\bar{h}_{i,t} | \bar{h}^{i, t-1}, \bar{m}^{i,t})$ for all
$i\in \mathcal{F}_{p}$ and $1\le t\le T$.

When the reward function is a path-based function, a straightforward observation from \Cref{thm:solution}
is that, the resulting machine policy $\hat{Q}(m_t|h^{t-1}, m^{t-1})$ is uniformly random if
$(h^{t-1}, m^{t-1}) \ne (\bar{h}^{r, t-1}, \bar{m}^{r, t-1})$. The machine policy along the support
of the reward function is induced by:
\begin{align}
Y_\gamma(\bar{m}_{r,t} | \bar{h}^{r, t-1}, \bar{m}^{r, t-1}) & = e^{\sum_{h_t}\hat{P}^{(n-1)}(h_t| \bar{h}^{r, t-1}, \bar{m}^{r, t})\log Y_\gamma((\bar{h}^{r, t-1}, h_t), \bar{m}^{r,t})} \nonumber\\
& = e^{(1 - P^*(\bar{h}_{r,t} | \bar{h}^{r,t-1}, \bar{m}^{r,t}))\log Y_\gamma^t + P^*(\bar{h}_{r,t} | \bar{h}^{r,t-1}, \bar{m}^{r,t})\log Y_\gamma(\bar{h}^{r, t}, \bar{m}^{r,t})},
\label{eq:Z-pathbased-reward}\end{align}
where $Y_\gamma^t := Y_\gamma(h^t, m^t)$ for $(h^t, m^t) \ne (\bar{h}^{r,t}, \bar{m}^{r,t})$. The sequence
of interactions can be suppressed because from \Cref{thm:solution} we can conclude that, after leaving the `profitable' path,
all $Y_\gamma$ will be the same, independent of corresponding $P(h^T\|m^T)$. From Eq. (\ref{eq:Z-pathbased-reward})
we can prove by induction that $Y_\gamma(\bar{m}_{r,t} | \bar{h}^{r, t-1}, \bar{m}^{r, t-1})$ does not change
after the first iteration. Thus the resulted machine's policy $\{\hat{Q}^{(n)}\}$ converges after the first iteration.

{\color{black}
\section{Proof of Theorem \ref{thm:convergence}}\label{proof:thm-convergence}
The solution to the $n$th human
estimation step can be written as $\hat{P}^{(n)} = h^*(\hat{Q}^{(n)},
\mathbf{c}_f(\hat{Q}^{(n)}), \mathbf{c}_g(\hat{Q}^{(n)},\hat{P}^{(n-1)}))$.
Indeed, $\mathbf{c}_f(\hat{Q}^{(n)})
= E_{P^*\hat{Q}^{(n)}}[\mathbf{f}(H^T, M^T)]$
depends on the true human behavior $P^*$, the feature set $\mathcal{F}$,
and also the machine policy in use $\hat{Q}^{(n)}$.
However, throughout AREA iterations, $P^*$ and $\mathcal{F}$
are fixed. Thus for simplicity we write $\mathbf{c}_f$ as
a function of $\hat{Q}^{(n)}$.
Similarly, we write $\mathbf{c}_g$ as a function of $\hat{Q}^{(n)}$ and $\hat{P}^{(n-1)}$,
where the only dependency on $\hat{P}^{(n-1)}$ is through the step dependent feature $c_g^{0,(n)}$
we have introduced.
Moreover, a direct result of Lemma 2 in \cite{ZBD13} showed that
$\mathbf{c}_g(\hat{Q}^{(n)},\hat{P}^{(n-1)})$ is actually a function of $\mathbf{Y}_\gamma^{(n)}$,
which is the $\mathbf{Y}_\gamma$ associated with $\hat{Q}^{(n)}$ as
defined in \Cref{thm:solution-machine}.

\begin{lemma}\label{lemma:y-sufficient}
During the $n$th iteration of AREA, let us denote the $\mathbf{Y}_\gamma$ in the
machine optimization problem by $\mathbf{Y}_\gamma^{(n)}$. Then
\begin{eqnarray*}
c_g^{0,(n)} = E_{\hat{P}^{(n-1)}\hat{Q}^{(n)}}\left[-\log \hat{Q}^{(n)}(M^T\|H^T) + \gamma r(H^T, M^T)\right]
= \log \sum_{m_1 \in \mathcal{M}}{Y}_\gamma^{(n)}(m_1)
\end{eqnarray*}
\end{lemma}

\begin{proof}
This is just a special case of Lemma 2 in \cite{ZBD13}. By plugging the recursive form
defined in \Cref{thm:solution-machine} we can prove it is true.
\end{proof}

Therefore $\hat{P}^{(n)}$ is actually a function of $\mathbf{Y}_\gamma^{(n)}$,
because $\hat{Q}^{(n)}$ is naturally a function of $\mathbf{Y}_\gamma^{(n)}$
by \Cref{thm:solution-machine}, and $\mathbf{c}_g$ is independent
of $\hat{P}^{(n-1)}$ given $\mathbf{Y}_\gamma^{(n)}$ by Lemma \ref{lemma:y-sufficient}:
$$
\hat{P}^{(n)} = h^*(\hat{Q}^{(n)}, \mathbf{c}_f(\hat{Q}^{(n)}), \mathbf{c}_g(\hat{Q}^{(n)},\hat{P}^{(n-1)}))
= h^*(m^*(Y_\gamma^{(n)}), \mathbf{c}_f(m^*(\mathbf{Y}_\gamma^{(n)})), \mathbf{c}_g(\mathbf{Y}_\gamma^{(n)})).
$$
In order to show convergence, it will be easier
to study it in terms of the underlying variables $\mathbf{Y}_\gamma^{(n)}$.
In the sequel when there is no ambiguity we will denote it by $\hat{P}^{(n)} = h^*(\mathbf{Y}_\gamma^{(n)})$.

Let us define the following function of $Y_\gamma$:
\begin{equation}\label{eq:def-L}
L(\mathbf{Y}_\gamma) := \mathbb{H}_{h^*(\mathbf{Y}_\gamma)m^*(\mathbf{Y}_\gamma)}(M^T\|H^T)
+ \gamma E_{h^*(\mathbf{Y}_\gamma)m^*(\mathbf{Y}_\gamma)}[r(H^T, M^T)],
\end{equation}
and $L^{(n)} = L(\mathbf{Y}_\gamma^{(n)})$. Now we
are ready to prove \Cref{thm:convergence}.

\begin{proof}
In order to show the convergence of $\{L(\mathbf{Y}_\gamma^{(n)})\}$,
we define the following functions of $\mathbf{Y}_\gamma$:
\begin{enumerate}
% \item $S({P}(M^T \parallel H^T))$: The solution to the human model estimation given the machine's policy ${P}(M^T \parallel H^T)$. For simplicity we will write it as $S(P)$.
\item $c(\mathbf{Y}_\gamma|\mathbf{Y}_\gamma^\prime)$ is 
the objective function of the machine's optimization problem,
where $\mathbf{Y}_\gamma$ and $\mathbf{Y}_\gamma^\prime$ are 
as defined in Theorem \ref{thm:solution-machine}
and are associated with
$Q(m^T\|h^T)$ and previous machine's policy $Q^\prime(m^T\|h^T)$,
\begin{eqnarray}
c(\mathbf{Y}_\gamma|\mathbf{Y}_\gamma^\prime)
& := & \mathbb{H}_{h^*(\mathbf{Y}_\gamma^\prime)m^*(\mathbf{Y}_\gamma)} (M^T\|H^T)  + \gamma  E_{h^*(\mathbf{Y}_\gamma^\prime)m^*(\mathbf{Y}_\gamma)}[ r(H^T,M^T)].
\end{eqnarray}

\item $L(\mathbf{Y}_\gamma)$ can be written as
\begin{eqnarray}
L(Y_\gamma) := c(\mathbf{Y}_\gamma|\mathbf{Y}_\gamma)
 %=  \mathbb{H}_{h(Y_\gamma)m(Y_\gamma)} (M^T\| H^T)  + \gamma  E_{h(Y_\gamma)m(Y_\gamma)}[ r(H^T,M^T)],
\end{eqnarray}
\end{enumerate}

During the AREA algorithm there are two possible cases: (1) $\hat{Q}^{(n+1)} = \hat{Q}^{(n)}$, and
(2) $\hat{Q}^{(n+1)} \ne \hat{Q}^{(n)}$. In case (1) it's straightforward that $\hat{Q}^{(m)}$ will
be the same as $\hat{Q}^{(n)}$, for all $m \ge n$. In case (2) we can show the convergence by
proving the strict monotonicity of $\{L(\hat{Y}_\gamma^{(n)})\}$ as follows.
\begin{eqnarray}
L(\mathbf{Y}_\gamma^{(n+1)}) & \ge & c(\mathbf{Y}_\gamma^{(n+1)} | \mathbf{Y}_\gamma^{(n)}) \label{eq:conv:L-and-l}\\
& \ge & c(\mathbf{Y}_\gamma^{(n)} | \mathbf{Y}_\gamma^{(n)}) \label{eq:conv:l-and-l}\\
& = & L(\mathbf{Y}_\gamma^{(n)}) \label{eq:conv:l-and-L}
\end{eqnarray}

Here Eq. (\ref{eq:conv:l-and-l}) follows from the optimality of the solution to the
machine's optimization (\ref{eqn:step2}), and Eq. (\ref{eq:conv:l-and-L}) follows by the definition
of $L(\mathbf{Y}_\gamma^{(n)})$. Thus we only need to show Eq. (\ref{eq:conv:L-and-l}).
Based on the definitions of the associated quantities, we have for all 
feasible $\mathbf{Y}_\gamma^{(n+1)}$:
\begin{eqnarray*}
L(\mathbf{Y}_\gamma^{(n+1)}) - c(\mathbf{Y}_\gamma^{(n+1)} | \mathbf{Y}_\gamma^{(n)})
& = & E_{h^*(\mathbf{Y}_\gamma^{(n+1)})m^*(\mathbf{Y}_\gamma^{(n+1)})}\left[-\log Q^{(n+1)}(M^T\|H^T) + \gamma r(H^T,M^T)\right] \\
& & - g_0(\mathbf{Y}_\gamma^{(n+1)})\\
&\ge & 0.
\end{eqnarray*}
The inequality holds true because in the human estimation problem, we introduced the constraint
$E_{P\hat{Q}^{(n)}}[g^{0,(n)}(H^T, M^T)] \ge c_g^{0,(n)}$.
Also, due to the boundedness of both $\mathbb{H}_{h^*(\mathbf{Y}_\gamma)m^*(\mathbf{Y}_\gamma)}(M^T\|H^T)$
and the expected reward function, $L(\mathbf{Y}_\gamma)$ is also upper bounded.
Therefore, the sequence generated by AREA recursion $\{L(\mathbf{Y}_\gamma^{(n)})\}$ converges
monotonically.
\end{proof}

An interesting observation we can make is that,
$\{L(\mathbf{Y}_\gamma^{(n)})\}$ converges to a value associated
with a fixed point of AREA iterations.
\begin{theorem}\label{thm:convergence-fixedpoint}
$\{L(\mathbf{Y}_\gamma^{(n)})\}$ converges to $L^\infty$, and there exists
a $\mathbf{Y}_\gamma^\infty$ such that $L(\mathbf{Y}_\gamma^\infty) = L^\infty$,
and $\mathbf{Y}_\gamma^\infty$ is a fixed point of AREA iterations,
i.e., $m^*(m( h^*(\mathbf{Y}_\gamma^\infty) )) = m^*(\mathbf{Y}_\gamma^\infty)$.
\end{theorem}
\begin{proof}
Now if we let the AREA algorithm stops once we observe $\hat{Q}^{(n+1)} = \hat{Q}^{(n)}$,
otherwise proceed to the next iteration,
then throughout the iterations of AREA (except for the last step when we stop),
machine optimization problem is strongly concave, thus obtain a unique maximum at any $n+1$st step,
which is $\hat{Q}^{(n+1)} \ne \hat{Q}^{(n)}$.
Therefore, Eq. (\ref{eq:conv:l-and-l}) holds true strictly.
Then $L(\mathbf{Y}_\gamma^{(n+1)}) > L(\mathbf{Y}_\gamma^{(n)})$ in case (2).
We can follow the result in \cite{Zan69},
by defining the solution set as the set of $\mathbf{Y}_\gamma$ such that
$
m^*(m( h^*(\mathbf{Y}_\gamma) )) = m^*(\mathbf{Y}_\gamma)
$, i.e., the
set of fixed point of AREA iterations, Corollary 1-1 in \cite{Zan69}
shows that one of the following statement is true:
\begin{enumerate}
\item The iteration stops in finite steps. Then we know it corresponds to the
case where we have for some $n$, $\hat{Q}^{(n + 1)} = \hat{Q}^{(n)}$. Thus $\forall~m > n$,
$\hat{Q}^{(m)} = \hat{Q}^{(n)}$, implying $\{\hat{Q}^{(n)}\}$ converges.
\item The iteration does not stop. Then according to Corollary 1-1 in \cite{Zan69},
any convergent subsequence of $\{\mathbf{Y}_\gamma^{(n)}\}$, say $\{\hat{\mathbf{Y}}_\gamma^{(k)}: k\in \mathcal{K}_j\subseteq \mathbb{Z}^+\}$
converges to an accumulation point $\hat{\mathbf{Y}}_\gamma^{(\infty),j}$
as $k \rightarrow \infty$, such that $\hat{\mathbf{Y}}_\gamma^{(\infty),j}$ is within the solution set.
\end{enumerate}

Therefore, due to the convergence of $\{L(\mathbf{Y}_\gamma^{(n)})\}$, all the accumulation points
of $\{\mathbf{Y}_\gamma^{(n)}\}$ have the same value of $L(\mathbf{Y}_\gamma)$ function, and are fixed points
of AREA iterations.
\end{proof}
}

\section{One important special case: Decomposable Features}\label{sec:decomposable}
In this section we discuss AREA under a special family of features. Specifically, we will derive performance
guarantees for the case where the solution has a special structure.

From now on we shall make the following assumption.
\begin{assumption}\label{assm:decomposable}
Reward function $r(h^T, m^T)$ is also used as a feature function in the estimation phase. 
Also, $\forall i \in \mathcal{F}$, $f^i(h^T,m^T)$ is decomposable, including the reward function $r(h^T,m^T)$, and $\forall i \in \mathcal{G}$, $g^i(h^T, m^T)$ is also decomposable.
\end{assumption}

\begin{lemma}
Under Assumption \ref{assm:decomposable}, the solution to the machine's optimization phase has no dependency
across time $t$. That is, at the $n$th iteration:
$$
\hat{Q}^{(n)}(m_t|h^{t-1}, m^{t-1}) = \hat{Q}^{(n)}(m_t).
$$
Moreover,
$$
\hat{Q}^{(n)}(m_t) \propto e^{\gamma E_{\hat{P}^{(n-1)}\hat{Q}^{(n)}}[r(H_t,m_t)]}.
$$
Note that under such assumptions, $E_{\hat{P}^{(n-1)}\hat{Q}^{(n)}}[r(H_t,m_t)]$ only depends on $\hat{P}^{(n-1)}$.
\end{lemma}

\begin{proof}
This can be proved in a manner similar to Lemma \ref{cor:decomposable}. Specifically, we can show that when Assumption \ref{assm:decomposable}
is true,
\begin{equation}
Y_\gamma (m_t|m^{t-1}, h^{t-1}) = (\prod _{\tau = t+1}^T Y_{\gamma,\tau}) e^{\gamma \sum_{\tau=1}^{t-1}r_\tau(h_\tau, m_\tau)}
    e^{\gamma\sum_{h_t}\hat{P}^{(n-1)}(h_t|m_t)r_t(h_t, m_t)},
\end{equation}
where $Y_{\gamma,t}:= \sum_{m_t}e^{\gamma \sum_{h_t}\hat{P}^{(n-1)}(h_t|m_t)r_t(h_t, m_t)}$.

The Markov property of $\hat{P}^{(n-1)}$ follows from Lemma \ref{cor:decomposable} and this identity holds true trivially when $t = T$, and can be proved by induction for other cases.
\end{proof}

Suppose our task is to find a machine's policy associated with a $\mathbf{Y}_\gamma$ to maximize
$L(\mathbf{Y}_\gamma)$ defined in Eq. (\ref{eq:def-L}).
In general, such an objective function is not well-defined in $Q$
because $Y_\gamma$ is not a function of $Q$.
However, when Assumption \ref{assm:decomposable} takes effect,
the causally conditional entropy is not dependent on $\hat{P}^{(n-1)}$:
\begin{align*}
\mathbb{H}_{\hat{P}^{(n-1)}Q}(M^T\|H^T) & = E_{\hat{P}^{(n-1)}Q}\left[-\log Q(M^T\|H^T)\right]\\
 & = E_{\hat{P}^{(n-1)}Q}\left[-\log \prod_{t = 1}^T Q_t(M_t)\right]\\
 & = \sum_{t = 1}^T E_{\hat{P}^{(n-1)}Q}\left[-\log Q_t(M_t)\right],
\end{align*}
where $E_{\hat{P}^{(n-1)}Q}[-\log Q_t(M_t)]$ actually does not depend on $\hat{P}^{(n-1)}$,
and we always have \\$E_{h^*(\mathbf{Y}_\gamma)Q}[r(H^T, M^T)] = E_{P^*Q}[r(H^T, M^T)]$.
In the sequel when Assumption \ref{assm:decomposable} is true we will use the notation
$\mathbb{H}_Q(M^T\|H^T)$ where $P$ is suppressed.
Then $L(\mathbf{Y}_\gamma)$ is actually a function of
$Q$, where $Q = m^*(\mathbf{Y}_\gamma)$  as it can be written as
$$
L(\mathbf{Y}_\gamma) =L(Q):= \mathbb{H}_Q(M^T\|H^T) + \gamma E_{P^*Q}[r(H^T, M^T)].
$$
%Then in the human estimation problem of $n$th iteration, $g_0(h^T, m^T)$ can be revised as
%\begin{equation}
%g_0(h^T, m^T) = r(h^T, m^T) = \sum_{t=1}^T r_t(h_t, m_t),
%\end{equation}
%and
%\begin{equation}
%\underline{g}_0 = \sum_{t = 1}^T E_{\hat{P}^{(n-1)}\hat{Q}^{(n)}}[r_t(h_t, m_t)].
%\end{equation}
And still we are able to show the strict monotonicity of $\{L(\hat{Q}^{(n)})\}$.

Moreover, we can show such objective function is indeed concave.
\begin{theorem}
When Assumption \ref{assm:decomposable} is true, $L(Q)$
is strongly concave with parameter $|\mathcal{M}|^T$ in $Q(m^T\|h^T)$.
\end{theorem}
\begin{proof}
It's easy to observe that $E_{P^*Q}[r(H^T, M^T)]$ is affine.
We already know that the causally conditional entropy term is strongly concave in $Q$
when $\hat{P}^{(n-1)}$ is fixed. Now we know that when Assumption \ref{assm:decomposable}
is true, the causally conditional entropy term is independent of $\hat{P}^{(n-1)}$.
Then it is a strong concave function in $Q$.
\end{proof}

\begin{theorem} \label{thm:optimality}
When Assumption \ref{assm:decomposable} is true, $\{L^{(n)}\} := \{L(\hat{Q}^{(n)})\}$
converges
to some limit $L^\infty$. If $Q^*$ is the global maximizer of
\begin{equation}\label{eq:max-L}
\max_{Q(m^T\|h^T)}~L(Q),
\end{equation}
then
\begin{equation}
L(Q^*) - L^\infty \le \gamma^2|\mathcal{M}|^{2T}r_{max}.
\end{equation}
\end{theorem}

\begin{proof}
First, according to Theorem \ref{thm:convergence-fixedpoint}, $L^\infty$ must be
$L(Y^\infty_\gamma)$ where $Y^\infty_\gamma$ is a fixed point.

The only difference between Eq.(\ref{eqn:step2}) and Eq.(\ref{eq:max-L}) is that
in Eq.(\ref{eq:max-L}), the mean reward is induced by $h(Q)$ which is a function
of $Q$ and in Eq.(\ref{eqn:step2}), that is induced by $\hat{P}$ which is fixed.
The gradient of $L(Q)$ is given by:
\begin{eqnarray*}
\lefteqn{\frac{\partial L(Q)}{\partial Q(m^T\|h^T)} = \frac{\partial \mathbb{H}_Q(M^T\|H^T)}{\partial Q(m^T\|h^T)}}\\
& & + \gamma (\left.\frac{\bm\partial E_{PQ}[r(H^T, M^T)]}{\bm\partial P(h^T\|m^T)}\right\vert_{P = h^*(Q)}\cdot \frac{\bm\partial h^*(Q)}{\partial Q(m^T\|h^T)} + \left.\frac{\partial E_{PQ}[r(H^T, M^T)]}{\partial Q(m^T\|h^T)}\right\vert_{P = h^*(Q)}),
\end{eqnarray*}
Here we suppress the human model $\hat{P}$ in the entropy term because the entropy
is independent of the human model.

And for the Eq.(\ref{eqn:step2}) at a fixed point, we have:
\begin{equation}
\frac{\partial \mathbb{H}_Q(M^T\|H^T)}{\partial Q(m^T\|h^T)} + \gamma \left.\frac{\partial E_{PQ}[r(H^T, M^T)]}{\partial Q(m^T\|h^T)}\right\vert_{P = h^*(Q)} = 0.
\end{equation}
Thus at the fixed point, i.e. when $Q = Q^\infty$,
$$
\frac{\partial L(Q)}{\partial Q(m^T\|h^T)} = \gamma \left.\frac{\bm\partial E_{PQ}[r(H^T, M^T)]}{\bm\partial P(h^T\|m^T)}\right\vert_{P = h^*(Q)}\cdot \frac{\bm\partial h^*(Q)}{\partial Q(m^T\|h^T)}.
$$

Also, from the moment-matching constraint and Assumption \ref{assm:decomposable} we know $E_{h^*(Q)Q}[r(H^T, M^T)] = E_{P^*Q}[r(H^T, M^T)]$.
Thus we have
\begin{align*}
\left.\frac{\bm\partial E_{PQ}[r(H^T, M^T)]}{\bm\partial P(h^T\|m^T)}\right\vert_{P = h^*(Q)}\cdot \frac{\bm\partial h^*(Q)}{\partial Q(m^T\|h^T)}
& = \left.\frac{\partial E_{PQ}[r(H^T, M^T)]}{\partial Q(m^T\|h^T)}\right\vert_{P = P^*} - \left.\frac{\partial E_{PQ}[r(H^T, M^T)]}{\partial Q(m^T\|h^T)}\right\vert_{P = h(Q)}\\
& = (P^*(h^T\|m^T) - P_Q(h^T\|m^T))r(h^T, m^T),
\end{align*}
which is also the gradient of $L(Q)$ at the fixed point.

Then according to the strong concavity, we have
\begin{align*}
L(Q^*) - L(Q^\infty) &\le \frac{|\mathcal{M}|^T}{2}\left(\gamma \sum_{h^T, m^T}(P^*(h^T\|m^T) - P_Q(h^T\|m^T))r(h^T, m^T)\right)^2 \\
& \le \gamma^2|\mathcal{M}|^{2T}r_{max}
\end{align*}
\end{proof}

{\color{black}
\section{Numerical evaluation set-up}\label{sec:evaluation-setup}
\subsection{Leaky, Competing Accumulator}
In the simulation set-up we use a discrete-time version of
the original continuous-time version devised in \cite{UsM01}.
The Leaky, Competing Accumulator model consists of a set of
accumulators $X_t(h)$ for $h\in\mathcal{H}$ at time $t$, representing
the tendency of picking $h$. The evolution of $X_t(h)$ is driven by
following parameters:
(1) A self decay coefficient $\alpha$, capturing the forgetting effect
of human memory;
(2) An inhibitory coefficient $\beta$, capturing the negative impact
of the belief in one option to others;
(3) Intensity/strength of the external stimuli, $\rho$, modeling the amount
of increment an external stimulus can bring to the associated accumulator.
(4) Power of noise $\sigma^2$, modeling the randomness in human decisions.
At each time $t$, the recursion of accumulators is given by
\begin{equation}
\forall h \in\mathcal{H},~~X_{t+1}(h) = \max(0, X_{t}(h) - \alpha X_{t}(h) - \beta \sum_{h^\prime \ne h}X_t(h^\prime)
 + \rho \mathbf{1}_{\{S_t = h\}} + \sigma N_{t,h}),
\end{equation}
where $S_t$ stands for the external stimulus at time $t$, and $N_{t,h}$ is an i.i.d. Gaussian noise.
Then the human will pick the action with the highest
value of $X_t(h)$ at $t$.

In our setting we use $\alpha = 0.1, \beta = 0.2, \rho = 0.4, \sigma^2 = 0.09$, and
the accumulators are all initialized at 0, so at the very beginning human
pick responses uniformly randomly.

\subsection{Q-learning}
The detailed update rule for $Q$ function is as follows.
After picking $m_t$ and observing human's response $h_t$ at time $t$, we do
\begin{align*}
\lefteqn{Q((h_{t- \tau}, \ldots, h_{t-1}, m_{t- \tau}, \ldots, m_{t-1}), t , m_{t}) \leftarrow}\\
& && (1 - \alpha) Q((h_{t- \tau}, \ldots, h_{t-1}, m_{t- \tau}, \ldots, m_{t-1}), t , m_{t})
+ \alpha \left(r_t(h_t, m_t) \right.\\
& && \left.+ \delta \max_{m\in\mathcal{M}} Q((h_{t- \tau + 1}, \ldots, h_{t}, m_{t- \tau + 1}, \ldots, m_{t}), t , m) \right).
\end{align*}
The $\alpha$ in Q-learning is the learning rate and $\delta$ is the discount factor to
balance the weight between current and future reward. In our evaluation,
we set $\alpha = 0.1$ and $\delta = 0.8$, which are values commonly used.

In our simulations,
the Q-learning picks its action according softmax of the associated
$Q$ function, which means when it observes the latest interactions
$(h_{t - \tau + 1}, \ldots, h_t, m_{t-\tau + 1}, \ldots, m_t)$ and $t + 1$,
it picks a response $m \in \mathcal{M}$ with probability
$$
\frac{e^{cQ((h_{t - \tau + 1}, \ldots, h_t, m_{t-\tau + 1}, \ldots, m_t) , t+1, m_{t+1})}}{\sum_{m^\prime_{t+1}} e^{cQ((h_{t - \tau + 1}, \ldots, h_t, m_{t-\tau + 1}, \ldots, m_t) , t+1, m^\prime_{t+1})}}.
$$
In our simulation we pick $c = 10$ so that Q-learning achieves a comparable
average reward as AREA after first 100 samples.
}